%% file: mainArXiv-vts.tex
\documentclass[11pt]{article}
\usepackage[margin=1in]{geometry}

\usepackage{amssymb}
\usepackage{latexsym}
\usepackage{amsmath}
\usepackage{amsthm}
\usepackage{thmtools}
\usepackage{thm-restate}
\usepackage{dirtytalk}
\usepackage{dsfont}
\usepackage{enumerate}
\usepackage{listings}
\usepackage{mathtools}
\usepackage{xcolor}
\usepackage[round]{natbib}
\usepackage{nameref}
\usepackage[colorlinks, linkcolor=blue, filecolor = blue, citecolor = blue, urlcolor = magenta]{hyperref}
\usepackage[noabbrev,capitalize]{cleveref}
\usepackage{graphicx}
\usepackage{mathtools}
\usepackage{xspace}

\usepackage[extdef=true]{delimset}
\usepackage{enumitem}
\usepackage{algorithm}
\usepackage{algorithmicx}
\usepackage[noend]{algpseudocode}
\usepackage{nicefrac}

\usepackage[textsize=tiny]{todonotes}

\allowdisplaybreaks
\input{macros.tex}

\title{\papertitle}

\author{}
\author{
Aadirupa Saha%
\thanks{Apple ML Research. This work started when the author was at TTI, Chicago; {\tt aadirupa@iisc.ac.in}.}
\and 
Branislav Kveton%
\thanks{AWS AI Labs. This work started prior to joining AWS AI Labs; {\tt bkveton@amazon.com}.} 
}
\date{}

\begin{document}

\maketitle

\input{abstract.tex}

\input{intro.tex}


\input{problem.tex}


\input{bayes_known.tex}

\input{bayes_unknown.tex}
\input{experiments.tex}

\input{conclusions.tex}

\newpage

\bibliographystyle{plainnat}
\bibliography{References}

\input{appendix.tex}

\end{document}

%% file: macros.tex
\newcommand{\papertitle}{{\begin{center}Only Pay for What Is Uncertain: \\Variance-Adaptive Thompson Sampling\end{center}}}

\usepackage[textsize=tiny]{todonotes}
\newcommand{\todob}[2][]{\todo[color=red!20,size=\tiny,inline,#1]{B: #2}}

\newtheorem{thm}{Theorem}
\newtheorem{lem}[thm]{Lemma}

\newcommand{\R}{{\mathbb R}}

\newcommand{\N}{{\mathbb N}}
\renewcommand{\P}{{\mathbb P}}

\newcommand{\cN}{{\mathcal N}}

\newcommand{\hmu}{\hat \mu}

\renewcommand{\v}{{\mathbf v}}

\renewcommand{\ref}{\cref}

\newcommand{\bkappa}{\boldsymbol \kappa}

\newcommand{\bbeta}{\boldsymbol \beta}
\newcommand{\balpha}{\boldsymbol \alpha}

\newcommand{\bnu}{{\boldsymbol \nu}}
\newcommand{\bmu}{{\boldsymbol \mu}}

\newcommand{\bsigma}{\boldsymbol \sigma}

\DeclareMathOperator*{\argmax}{argmax}

\mathchardef\mhyphen="2D

\newcommand{\red}[1]{\textcolor{red}{{#1}}}

\newcommand{\klucb}{\ensuremath{\tt KL\mhyphen UCB}\xspace}

\newcommand{\ucb}{\ensuremath{\tt UCB1}\xspace}
\newcommand{\ucbnormal}{\ensuremath{\tt UCB1\mhyphen Normal}\xspace}
\newcommand{\ucbtuned}{\ensuremath{\tt UCB1\mhyphen Tuned}\xspace}
\newcommand{\ucbv}{\ensuremath{\tt UCB\mhyphen V}\xspace}

\newcommand{\varts}{\ensuremath{\tt VarTS}\xspace}

\newcommand{\htts}{\ensuremath{\tt TS14}\xspace}
\newcommand{\ztts}{\ensuremath{\tt TS20}\xspace}

\newcommand{\realset}{\mathbb{R}}

\newcommand{\diag}[1]{\mathrm{diag}\left(#1\right)}

\newcommand{\E}[1]{{\mathbb{E} \left[#1\right]}}
\newcommand{\condE}[2]{\mathbb{E} \left[#1 \,\middle|\, #2\right]}

\newcommand{\condprob}[2]{\mathbb{P} \left(#1 \,\middle|\, #2\right)}

\newcommand{\std}[1]{\mathrm{std} \left[#1\right]}

\newcommand{\bign}[1]{\big(#1\big)}
\newcommand{\biggn}[1]{\bigg(#1\bigg)}
\newcommand{\Bign}[1]{\Big(#1\Big)}

\newcommand{\bigsn}[1]{\big[#1\big]}
\newcommand{\biggsn}[1]{\bigg[#1\bigg]}
\newcommand{\Bigsn}[1]{\Big[#1\Big]}

\newcommand*\dif{\mathop{}\!\mathrm{d}}

\newcommand{\I}[1]{\mathds{1} \! \left\{#1\right\}}

\renewcommand{\set}[1]{\left\{#1\right\}}



%% file: abstract.tex
\begin{abstract}
Most bandit algorithms assume that the reward variances or their upper bounds are known, and that they are the same for all arms. This naturally leads to suboptimal performance and higher regret due to variance overestimation. On the other hand, underestimated reward variances may lead to linear regret due to committing early to a suboptimal arm. This motivated prior works on variance-adaptive frequentist algorithms, which have strong instance-dependent regret bounds but cannot incorporate prior knowledge on reward variances. We lay foundations for the Bayesian setting, which incorporates prior knowledge. This results in lower regret in practice, due to using the prior in the algorithm design, and also improved regret guarantees. Specifically, we study Gaussian bandits with \emph{unknown heterogeneous reward variances}, and develop a Thompson sampling algorithm with prior-dependent Bayes regret bounds. We achieve lower regret with lower reward variances and more informative priors on them, which is precisely why we pay only for what is uncertain. This is the first result of its kind. Finally, we corroborate our theory with extensive experiments, which show the superiority of our variance-adaptive Bayesian algorithm over prior frequentist approaches. We also show that our approach is robust to model misspecification and can be applied with estimated priors.
\end{abstract}

%% file: intro.tex
\section{Introduction}
\label{sec:introduction}

A \emph{stochastic bandit} \citep{lai85asymptotically,auer02finitetime,lattimore19bandit} is an online learning problem where a \emph{learning agent} sequentially interacts with an environment over $n$ rounds. In each round, the agent pulls an \emph{arm} and receives a \emph{stochastic reward}. The mean rewards of the arms are initially unknown and the agent learns them by pulling the arms. Therefore, the agent faces an \emph{exploration-exploitation dilemma} when pulling the arms: \emph{explore}, and learns more about the arms; or \emph{exploit}, and pull the arm with the highest estimated reward. An example of this setting is a recommender system, where the arm is a recommendation and the reward is a click.

Most bandit algorithms assume that the reward variance or it is upper bound is known. For instance, the confidence intervals in \ucb \citep{auer02finitetime} are derived under the assumption that the rewards are $[0, 1]$, and hence $\sigma^2$-sub-Gaussian for $\sigma = 0.5$. In Bernoulli \klucb \citep{garivier11klucb} and Thompson sampling (TS) \citep{agrawal12analysis}, tighter confidence intervals are derived for Bernoulli rewards. Specifically, a Bernoulli random variable wither either a low or high mean also has a low variance. In general, the reward variance is hard to specify \citep{audibert09}. \emph{While overestimating it is typically safe, this decreases the learning rate of the bandit algorithm and increases regret. On the other hand, when the variance is underestimated, this may lead to linear regret because the algorithm commits to an arm without sufficient evidence.}

We motivate learning of reward variances by the following example. Take a movie recommender that learns to recommend highest rated movies in a \say{Trending Now} carousel. The movies are rated on scale $[1, 5]$. Some movies, such as The Godfather, are classics. Therefore, their ratings are high on average and have low variance. On the other hand, ratings of low-budget movies are often low on average and have low variance, due to the quality of the presentation. Finally, most movies are made for a specific audience, such as Star Wars, and thus have a high variance in ratings. Clearly, any sensible learning algorithm would require a lot less queries to estimate the mean ratings of movies with low variances. Since the variance is unknown a priori, adaptation is necessary. This would reduce the overall query complexity and improve statistical efficiency--as one should--because we only pay for what is uncertain. This example is not limited to movies and applies to other domains, such as online shopping. Our work answers the following questions in affirmative:

\begin{center}
\emph{
Can we quickly learn the right representation of the reward distribution for efficient learning? What is the right trade-off of the learner's performance (regret) versus the prior parameters and reward variances? Can we design an algorithm to achieve that rate? Does the regret decrease with lower reward variances and more informative priors on them?}
\end{center}

Unknown reward variances are a major concern and thus have been studied extensively. In the cumulative regret setting, \citet{audibert09} proposed an algorithm based on upper confidence bounds (UCBs) and \citet{mukherjee18} proposed an elimination algorithm. In best-arm identification (BAI) \citep{audibert09exploration,bubeck09pure}, several papers studied the fixed-budget \citep{gabillon11multibandit,faella20kernel,SGV20,lalitha2023fixed} and fixed-confidence \citep{lu21variancedependent,zhou22approximate,jourdan22dealing} settings with unknown reward variances. All above works studied frequentist algorithms. On the other hand, Bayesian algorithms based on posterior sampling \citep{thompson33likelihood,chapelle11empirical,agrawal12analysis,russo14learning,russo18tutorial,kveton21metathompson,hong22hierarchical} perform well in practice, but learning of reward variances in these algorithms is understudied. Our problem setup and Bayesian objective are presented in \cref{sec:prob}. Our contributions are summarized below.

\textbf{Contributions:} \textbf{(1)} To warm up, we study Thompson sampling in a $K$-armed Gaussian bandit with \emph{known heterogeneous reward variances} and bound its Bayes regret (\cref{sec:bayes_known}). Our regret bound (\cref{thm:bayes_known}) decreases as reward variances decrease. It also approaches zero as the prior variances of mean arm rewards go to zero. In this case, a Bayesian learning agent knows the bandit instance with certainty. \textbf{(2)} We propose a Thompson sampling algorithm \varts for a $K$-armed Gaussian bandit with \emph{unknown heterogeneous reward variances} (\cref{sec:bayes_unknown2}). \varts maintains a joint Gaussian-Gamma posterior for the mean and precision of the rewards of all arms and samples from them in each round. 
\textbf{(3)} We prove a Bayes regret bound for \varts (\cref{thm:bayes_unknown2}), which decreases with lower reward variances and more informative priors on them. This is the first regret bound of its kind. The novelty in our analysis is in handling random confidence interval widths due to random reward variances. The resulting regret bound captures the same trade-offs as if the variance was known, replaced by the corresponding prior-dependent quantities. \textbf{(4)} We comprehensively evaluate \varts on various types of reward distributions, from Bernoulli to beta to Gaussian. Our evaluation shows that \varts outperforms all existing baselines, even with an estimated prior (\cref{sec:experiments}). This showcases the generality and robustness of our method.

\section{Related Work}
\label{sec:rel}

The beginnings of variance-adaptive algorithms can be traced to \citet{auer02finitetime}. \citet{auer02finitetime} proposed a variance-adaptive \ucb, called \ucbnormal, for Gaussian bandits where the reward distribution of arm $i$ is $\cN(\mu_i, \sigma_i^2)$ and $\sigma_i > 0$ is assumed to be known. The $n$-round regret of this algorithm is $O\left(\sum_{i: \mu_i < \mu_{a^*}} \frac{\sigma_i^2}{\Delta_i} \log n\right)$, where $a^*$ is the arm with the highest mean reward $\mu_i$. The first UCB algorithm for unknown reward variances with an analysis was \ucbv by \citet{audibert09}. The key idea in the algorithm is to design high-probability confidence intervals based on empirical Bernstein bounds. The $n$-round regret of \ucbv is $O\left(\sum_{i: \mu_i < \mu_{a^*}} \left(\frac{\sigma_i^2}{\Delta_i} + b\right) \log n\right)$, where $b$ is an upper bound on the absolute value of the rewards. In summary, variance adaptation in \ucbv incurs only a small penalty of $O(b K \log n)$. \citet{mukherjee18} proposed an elimination-based variant of $\ucbv$ that attains the optimal gap-free regret of $O(\sqrt{K n})$, as opposing to the original $O(\sqrt{K n \log n})$. While empirical Bernstein bounds are general, they tend to be conservative in practice. This was observed before by \citet{garivier11klucb} and we observe the same trend in our experiments (\cref{sec:experiments}). Our work can be viewed as a similar development with Thompson sampling. We show that Thompson sampling with unknown reward variances (\cref{sec:bayes_unknown2}) incurs only a slightly higher regret than the one with known variances (\cref{sec:bayes_known}), by a multiplicative factor. Compared to \ucbv, the algorithm is highly practical.

Two closest papers to our work are \citet{honda14optimality,MV20}. Both papers propose variance-adaptive Thompson sampling and bound its regret. There are three key differences from our work. First, the algorithms of \citet{honda14optimality,MV20} are designed for the frequentist setting. Specifically, they have a fixed sufficiently-wide prior, and enjoy a per-instance regret bound under this prior. While this is a strong guarantee, the algorithms can perform poorly when priors are narrower and thus more informative. Truly Bayesian algorithm designs, as proposed in our work, can be analyzed for any informative prior. Second, the analyses of \citet{honda14optimality,MV20} are frequentist. Therefore, they cannot justify the use of more informative priors. In contrast, we prove regret bounds that decrease with lower reward variances and more informative priors on them. Finally, the regret bounds of \citet{honda14optimality,MV20} are asymptotic. We provide strong finite-time guarantees. We discuss these difference in more detail after \cref{thm:bayes_unknown2} and demonstrate them empirically in \cref{sec:experiments}.

Another related line of works are variance-dependent regret bounds for $d$-dimensional linear contextual bandits \citep{kim22,quan22,quan23,zhang21}. These works address the problem of time-dependent variance adaptivity. They derive frequentist regret bounds that scale as $\tilde O(\text{poly}(d)\sqrt{1 + \sum_{s = 1}^n \sigma_s^2})$, where $\sigma_s^2$ is an unknown reward variance in round $s$. This setting is different from ours in two aspects. First, they study changing reward variances over time but keep them fixed across the arms. We do the opposite in our work. Second, their algorithm designs and analyses are frequentist, and thus cannot exploit prior knowledge. On the other hand, we focus only on $K$-armed bandits, which is a special case of linear bandits.

%% file: problem.tex
\section{Problem Setup}
\label{sec:prob}

\textbf{Notation.} The set $\set{1, \dots, n}$ is denoted by $[n]$. The indicator $\I{E}$ denotes that event $E$ occurs. We use boldface letters to denote vectors. For any vector $\v \in \R^d$, we denote by $v_i$ its $i$-th entry, or sometimes simply $v(i)$. We denote the entry-wise square of $\v$ by $\v^2$. A diagonal matrix with entries $\v$ is denoted by $\diag{\v}$. $\tilde{O}$ denotes big O notation up to polylogarithmic factors. Gaussian, Gamma, and Gaussian-Gamma distributions are denoted by $\cN$, $\text{Gam}$, and $NG$, respectively.

\textbf{Setting.} A bandit \emph{instance} is a pair of mean arm rewards and reward variances, $(\bmu, \bsigma^2)$, where $\bmu \in \R^K$ is a vector of mean arm rewards, $\bsigma^2 \in \R_{\geq 0}^K$ are the reward variances, and $K$ is the number of arms. We propose algorithms and analyze them for both when the reward variances $\bsigma^2$ are known (\cref{sec:bayes_known}) and unknown (\cref{sec:bayes_unknown2}).

\textbf{Feedback model.} The agent interacts with the bandit instance $(\bmu, \bsigma^2)$ for $n$ rounds. In round $t \in [n]$, it pulls one arm and observes a stochastic realization of its reward. We denote the pulled arm in round $t$ by $A_t \in [K]$, a stochastic reward vector for all arms in round $t$ by $\boldsymbol{x}_t \in \realset^K$, and the reward of arm $i \in [K]$ by $x_{t, i} \in \realset$. The rewards are sampled from a Gaussian distribution, $x_{t, i} \sim \cN(\mu_i, \sigma_i^2)$. The interactions of the agent up to round $t$ are summarized by a \emph{history} $H_s = \big( A_{1}, x_{1,A_1}, \dots, A_{t}, x_{t, A_t}\big)$.

\textbf{Bayesian bandits setting.} We consider a \emph{Bayesian multi-armed bandit} \citep{russo14learning,russo18tutorial,kveton21metathompson,hong22hierarchical} where the bandit instance is either fully or partially random:
\textbf{(i).} When \emph{reward variances are known} (\cref{sec:bayes_known}): The bandit instance $(\bmu,\bsigma)$ is generated as follows. The mean arm rewards are sampled from a Gaussian distribution, $\bmu \sim P_0 = \cN(\bmu_0, \diag{\bsigma_0^2})$, where $\bmu_0 \in \realset^K$ and $\bsigma_0^2 \in \R_{\geq 0}^K$ are the prior means and variances of $\bmu$, respectively. Both $\bmu_0$ and $\bsigma_0^2$ are assumed to be known by the agent. The reward variances $\bsigma^2$ are also known.
\textbf{(ii).} When \emph{reward variances are unknown} (\cref{sec:bayes_unknown2}): We assume that the bandit instance $(\bmu, \bsigma)$ is sampled from a Gaussian-Gamma prior distribution. More specifically, for any arm $i$, the mean and variance of its rewards are sampled as $(\mu_i, \sigma_i^{-2}) \sim NG(\mu_{0, i}, \kappa_{0, i}, \alpha_{0, i}, \beta_{0, i})$, where $(\bmu_0, \bkappa_0, \balpha_0, \bbeta_0)$ are known prior parameters. This can also be seen as first sampling $\sigma_i^{-2} \sim \text{Gam}(\alpha_{0,i},\beta_{0,i})$ followed by $\mu_i \sim \cN(\mu_{0,i}, \frac{\sigma_i^2}{\kappa_{0,i}})$. This equivalence follows from the basic properties of the Gaussian-Gamma distribution, as shown in \cref{lem:gauss_gam} in Appendix.

\textbf{Regret.} We measure the $n$-round \emph{Bayes regret} of a learning agent with instance prior $P_0$ as:
\begin{align}
  \textstyle
  R_n
  =  \E{{\sum_{t = 1}^n
  \mu_{A^*} - n\mu_{A_t}}}\,,
  \label{eq:reg_bayes}
\end{align} 
where $A^* = \argmax_{i \in [K]} \mu_{i}$ denotes the \emph{optimal arm}. The above expectation is over the mean arm rewards $\bmu \sim P_0$, unlike in the frequentist setting where $\bmu$ would be unknown but fixed \citep{lattimore19bandit}. The randomness in the above expectation also includes how the algorithm chooses $A_t$ and the randomness in the observed bandit feedback $x_{t,A_t} \sim \cN(\mu_{A_t},\sigma_{A_t}^2)$.

We depart from the classic bandit setting \citep{auer02finitetime,abbasi-yadkori11improved,lattimore2020bandit} in two major ways. First, we consider Gaussian reward noise, as opposing to more general sub-Gaussian noise. The Gaussian noise and corresponding conjugate priors lead to closed-form posteriors in our algorithms and analyses, which simplifies them. This is why this choice has been popular in recent Bayesian analyses \citep{lu19informationtheoretic,kveton21metathompson,wan21metadatabased,hong22hierarchical,hong22deep}. Second, our regret is Bayesian, on average over bandit instances. An alternative would be the frequentist regret, which holds for any bounded bandit instance. We choose the Bayes regret because it can be used to capture the relation between the bandit instance and its prior, and thus show benefits of informative priors. We discuss this in depth throughout the paper, and especially after \cref{thm:bayes_known} and \cref{thm:bayes_unknown2}. To alleviate concerns about Gaussian posteriors in the algorithm design, we experiment with a variety of other bandit problems in \cref{sec:experiments}.

%% file: bayes_known.tex
\section{Gaussian Bandit with Known Variances}
\label{sec:bayes_known}

We start with the Bayesian setting with Gaussian rewards and known heterogeneous reward variances. In \cref{sec:bayes_known_algo}, we introduce a Thompson sampling algorithm \citep{thompson33likelihood,chapelle11empirical,agrawal12analysis,russo14learning,gopalan14thompson} for this setting. Gaussian TS is straightforward and appeared in many prior works, starting with \citet{agrawal13further}. We state the regret bound and discuss it in \cref{sec:known_reg}. The regret bound scales roughly as: 
$
  \textstyle
  \sqrt{n\log n} \sqrt{\sum_{i =1}^K\sigma_i^2
  \log\Big(1 + n\frac{\sigma_{0,i}^{2}}{\sigma_i^{2}}\Big)}\,.
  \label{eq:rough bound known}
$
One notable property of the bound is that it goes to zero when the reward variances $\sigma_i^2$ or the prior variances of the mean arm rewards $\sigma_{0, i}^2$ do. Although the bound is novel, its proof mostly follows \citet{kveton21metathompson}. The main reason for stating the bound is to contrast it with the main result in \cref{sec:bayes_unknown2}.

\input{algo_bayes_known.tex}

\subsection{Regret Analysis}
\label{sec:known_reg}

Before analyzing \cref{alg:bayes_known}, let us recall the setting again: The mean arm rewards are sampled from a Gaussian prior, $\bmu \sim P_0 = \cN(\bmu_0, \diag{\bsigma_0^2})$, where $\bmu_0 \in \realset^K$ and $\bsigma_0^2 \in \R_{\geq 0}^K$ are the prior means and variances of $\bmu$, respectively. The reward of arm $i$ in round $t$ is sampled as $x_{t,i} \sim \cN(\mu_i,\sigma_i^2)$. Both $\bmu_0$ and $\bsigma_0^2$ and reward variances $\bsigma^2$ are fixed and known. Our regret bound is presented below.

\begin{restatable}[Variance-dependent regret bound for known variances]{thm}{bayesknown}
\label{thm:bayes_known} Consider the above setting. Then for any $\delta > 0$, the Bayes regret of Gaussian TS is bounded as:
\begin{align*}
  & R_n
  \leq \sum_{i = 1}^K\sqrt{\frac{2 \sigma_{0,i}^2}{\pi}} n \delta + 
  \sqrt{2n} \sqrt{\sum_{i =1}^K\sigma_i^2\Big(\log(1 + n\sigma_{0,i}^{2}\sigma_i^{-2}) + \sigma_{0,i}^{2}\sigma_i^{-2}\Big)\log(1/\delta)}\,.
\end{align*}
\end{restatable}

The complete proof of \cref{thm:bayes_known} is in \cref{app:known}. We discuss the bound below.

\textbf{Dependence on all parameters of interest and prior.} For $\delta = 1 / n$, the bound in \cref{thm:bayes_known} scales roughly as
$
  \tilde{O}\left(\sqrt{n \sum_{i = 1}^K\sigma_i^2 \log\Bign{1 + n\sigma_{0,i}^{2}\sigma_i^{-2} }} + \sqrt{n \sum_{i = 1}^K \sigma_{0,i}^2}\right)
$.
Note that we ignore the first term in \cref{thm:bayes_known}, which is order-wise dominated by the second term when $\delta = 1 / n$. Our bound has several properties that we discuss next. First, it matches the usual $\sqrt{n}$ dependence of all classic Bayes regret bounds \citep{russo14learning,russo16information,lu19informationtheoretic}. Second, it increases with variances $\sigma_i^2$ of individual arm rewards, which is expected because higher reward variances make learning harder. Third, the bound can be viewed as a generalization of existing bounds that assume homogeneous reward variance. Specifically, \citet{kveton21metathompson} derived a $\tilde{O}(\sqrt{\sigma^2 K n})$ Bayes regret bound in Lemma 4 under the assumption that the reward distribution of arm $i$ is $\cN(\mu_i, \sigma^2)$. We match it when $\sigma_i = \sigma$ for all arms $i$. Fourth, the bound approaches zero as $\sigma_{0, i} \to 0$. In this setting, Gaussian TS knows the mean reward $\mu_i$ almost with certainty because its prior variance $\sigma_{0, i}^2$ is low, and no exploration is necessary. This is a unique property of Bayes regret bounds that is not captured by any frequentist analysis, such as that of \ucbnormal \citep{auer02finitetime}.

\textbf{Regret optimality.} Starting with the seminal works of \citet{russo14learning,russo16information}, all Bayes regret bounds are $\tilde{O}(\sqrt{n})$ and do not have finite-time instance-dependent lower bounds. \citet{lattimore2020bandit} derived a $\Omega(\sqrt{K n})$ asymptotic lower bound for a $K$-armed bandit as $n \to \infty$ (Theorem 35.1). Our regret bound (\cref{thm:bayes_known}) matches this rate when the prior and reward variances are the same for all arms, such as $\sigma_{0, i} = \sigma_i = 1$ for any $i \in [K]$. In addition, it gives an improved dependence on lower reward variances and more informative priors, which implies faster learning rates in these regimes. In fact, when the prior variances of all mean arm rewards go to zero, $\sigma_{0, i} \to 0$ for any $i \in [K]$, our bound goes to zero; as expected. Thus we conjecture that our regret bound is worst-case optimal. The only other lower bound that we are aware of is $\Omega(\log^2 n)$ for a $K$-armed bandit (Theorem 3 in \citet{lai87adaptive}). This lower bound is asymptotic and applies only to exponential-family reward distributions with a single parameter, which excludes Gaussian distributions because they have two parameters. To conclude, we believe that deriving a tight finite-time $\Omega({\sqrt{K n}})$ lower bound for our setting is an important problem, but this may require new techniques and should be of independent interest to the Bayesian community itself.

%% file: algo_bayes_known.tex
\begin{algorithm}[h]
  \caption{Gaussian TS for known reward variances.}
  \label{alg:bayes_known}
  \begin{algorithmic}[1]
    \State \textbf{Inputs:} Prior means $\bmu_0$, prior variances $\bsigma_0^2$, and reward variances $\bsigma^2$
    \State \textbf{Init:} $\forall i \in [K]: N_1(i):=0, \ \hmu_{1,i}:= \mu_{0,i}, \ \sigma_{1,i}:= \sigma_{0,i}, \bar{x}_{1,i}:= 0$ 
    \Statex \vspace{-0.05in}
    \For{$t = 1, \dots, n$} 
    		\State Posterior sampling: $\forall i \in [K]: \tilde \mu_{t,i} \sim \cN(\hmu_{t,i},\sigma_{t,i}^2)$
        \State Pull: $A_t:= \arg\max_{i \in [K]}\tilde \mu_{t,i}$
    		\State Reward feedback: $x_{t,A_t} \sim \cN(\mu_{A_t},\sigma_{A_t}^2)$
    		\State Posterior update: 
    		\For{$i = 1, \dots, K$}
			\State $N_{t+1}(i):= N_{t}(i) + \I{A_t = i}$, \ $\sigma_{t+1, i}^2:= \frac{1}{\sigma_0^{-2} + N_{t+1}(i) \sigma_i^{-2}}$
     \State $\hmu_{t+1, i}
  := \sigma_{t+1,i}^2\Big( \frac{\mu_{0,i}}{\sigma_{0,i}^2} + \frac{N_{t+1}(i)\bar x_{t+1,i}}{\sigma_{i}^2} \Big)$, s.t. $\bar x_{t+1,i}:= \frac{1}{N_{t+1}(i)}\sum_{s=1}^{t} \I{A_{s}=i} x_{s,i}$
    		\EndFor
	\EndFor
  \end{algorithmic}
\end{algorithm}

\vspace{-5pt}
\subsection{Gaussian Thompson Sampling}
\label{sec:bayes_known_algo}
\vspace{-2pt}
The key idea in our algorithm is to maintain a posterior distribution over the unknown mean arm rewards $\bmu$ and act optimistically with respect to samples from it. Since $\bmu$ and its rewards are sampled from Gaussian distributions, the posterior is also Gaussian. Specifically, the posterior distribution of arm $i$ in round $t$ is $\cN(\hmu_{t,i},\sigma_{t,i}^2)$, where $\hat \mu_{t, i}$ and $\sigma_{t, i}^2$ are the posterior mean and variance, respectively, of arm $i$ in round $t$. These quantities are initialized as $\hat \mu_{1,i}:= \mu_{0,i}$ and $\sigma_{1,i}:= \sigma_{0,i}$.

Our algorithm is presented in \cref{alg:bayes_known} and we call it \emph{Gaussian TS} due to Gaussian rewards. The algorithm works as follows. In round $t$, it samples the mean reward of each arm $i$ from its posterior, $\tilde \mu_{t, i} \sim \cN(\hmu_{t, i}, \sigma_{t, i}^2)$. After that, the arm with the highest posterior-sampled mean reward is pulled, $A_t:= \arg\max_{i \in [K]}\tilde \mu_{t,i}$. Finally, the algorithm observes a stochastic reward of arm $A_t$, $x_{t,A_t} \sim \cN(\mu_{A_t},\sigma_{A_t}^2)$, and updates its posteriors (\cref{lem:gaussian posterior update} in Appendix) as:
\begin{align*}
  \sigma_{t+1, i}^2
  := \frac{1}{\sigma_0^{-2} + N_{t+1}(i) \sigma_i^{-2}}\,, \quad
  \hmu_{t+1, i}
  := \sigma_{t,i}^2\bigg( \frac{\mu_{0,i}}{\sigma_{0,i}^2} +
  \frac{N_{t+1}(i)\bar x_{t+1,i}}{\sigma_{i}^2} \bigg)\,,
\end{align*}
where $\bar x_{t+1,i}:= \frac{1}{N_{t+1}(i)}\sum_{s=1}^{t} \I{A_{s}=i} x_{s,i}$ 
is the empirical mean reward of arm $i$ at the beginning of round $t+1$ and $N_{t+1}(i)$ is the number of its pulls.

%% file: bayes_unknown.tex
\section{Gaussian Bandit with Unknown Variances}
\label{sec:bayes_unknown2}

Our main contribution is that we consider the Bayesian setting with Gaussian rewards and unknown heterogeneous reward variances. Similarly to \cref{sec:bayes_known}, we propose a Thompson sampling algorithm for this setting in \cref{sec:bayes_unknown_algo}. 
We discuss the regret bound in \cref{sec:unknown_reg}. 
The regret bound scales roughly as:
$
  \textstyle
  \sqrt{n \log n}\sqrt{\sum_{i = 1}^K \frac{\beta_{0,i}}{\alpha_{0,i}-1}
  \log\Big( 1+\frac{n}{\kappa_{0,i}}\Big)}\,,
$
where $\frac{\beta_{0, i}}{\alpha_{0, i} - 1}$ represents a proxy for the reward variance $\sigma_i^2$ in \eqref{eq:rough bound known} and $\kappa_{0, i}^{-1}$ plays the role $\sigma_{0, i}^2/\sigma_i^2$. Since the dependencies are analogous, the bound captures the structure of the problem similarly to \eqref{eq:rough bound known}. \emph{Our main novelty lies in handling the uncertainty of reward variances $\bsigma^2$, which is unique among all existing TS proofs.}

\input{algo_bayes_unknown.tex}

\vspace{-5pt}
\subsection{Regret Analysis}
\label{sec:unknown_reg}

Recall from \cref{sec:prob}, here the bandit instance $(\bmu, \bsigma)$ is sampled from a Gaussian-Gamma distribution: For any arm $i$, the mean and variance of its rewards are sampled as $(\mu_i, \sigma_i^{-2}) \sim NG(\mu_{0, i}, \kappa_{0, i}, \alpha_{0, i}, \beta_{0, i})$, where $(\bmu_0, \bkappa_0, \balpha_0, \bbeta_0)$ are known prior parameters. This can also be seen as first sampling $\sigma_i^{-2} \sim \text{Gam}(\alpha_{0,i},\beta_{0,i})$ and then $\mu_i \sim \cN(\mu_{0,i}, \frac{\sigma_i^2}{\kappa_{0,i}})$. Our regret bound shows:

\begin{restatable}[Variance-dependent regret bound for unknown variances]{thm}{bayesunknown}
\label{thm:bayes_unknown2} Consider the above setting and let $\alpha_{0,i }\geq 1$ for all arms $i \in [K]$. Then for any $\delta > 0$, the Bayes regret of \varts is bounded as:
$
  R_n
  \leq C\sqrt{n \log(1 / \delta)} + \delta C \sqrt{\frac{nK}{2\pi}}\,,
$
where $C^2 = \sum_{i = 1}^K \frac{\beta_{0,i}}{\alpha_{0,i} - 1} \biggn{ \frac{2}{\kappa_{0,i}} + \frac{0.5}{\kappa_{0,i}(\alpha_{0,i}-1)} + 5\log\Big(1 + \frac{n}{\kappa_{0,i}}\Big) }$ is a constant dependent on prior parameters.
\end{restatable}

\textbf{Proof of \cref{thm:bayes_unknown2}.} The difficulty lies in tightly bounding confidence intervals of random reward means with unknown reward variances, to get the right dependence on reward variances $\frac{\beta_{0, i}}{\alpha_{0, i} - 1}$ and mean reward variances $\frac{1}{\kappa_{0, i}}$. This is algebraically challenging due to the complicated nature of the posterior updates of Gaussian-Gamma distributions, as presented in \cref{alg:bayes_unknown}. To overcome these difficulties, we carefully condition random variables on each other together with appropriate histories, and combine these using Jensen's and Cauchy-Schwarz inequalities. The key lemmas along with the complete proof of \cref{thm:bayes_unknown2} are in \cref{app:unknown}.

\textbf{Dependence on all parameters of interest and prior.} For $\delta = 1 / n$, the bound in \cref{thm:bayes_unknown2} is $\tilde{O}(\sqrt{C n})$. The dependence on $\sqrt{n}$ is the same as in \cref{thm:bayes_known}. Further, a closer examination of $C$ reveals many similarities: 

First, $\beta_{0, i} / (\alpha_{0, i} - 1)$ is the mean of an Inverse-Gamma distribution with parameters $(\alpha_{0, i}, \beta_{0, i})$. Since the precision of the reward distribution of arm $i$, $\lambda_i$, is sampled from $\text{Gam}(\alpha_{0, i}, \beta_{0, i})$, we have that $\beta_{0, i} / (\alpha_{0, i} - 1)$ is the mean of the reward variance distribution of arm $i$. Thus $\beta_{0, i} / (\alpha_{0, i} - 1)$ in \cref{thm:bayes_unknown2} plays the role of $\sigma_i^2$ in \cref{thm:bayes_known}, which represents the \emph{effective reward variance}.

Second, $\kappa_{0, i}$ in the Gaussian-Gamma prior plays the role of $\sigma_i^2 / \sigma_{0, i}^2$ in the known variance setting \citep{murphy2007conjugate}. Therefore, as $\kappa_{0, i} \to \infty$, the bound in \cref{thm:bayes_unknown2} should go to zero, similarly to \cref{thm:bayes_known}. This is indeed the case and a very unique property of Bayes regret bounds, which is not captured by \citet{honda14optimality,zhou22approximate}.

Finally, we take $\alpha_{0, i}, \beta_{0, i} \to \infty$ while keeping $\beta_{0, i} / (\alpha_{0, i} - 1)$ constant. As the mean and variance of the Inverse-Gamma distribution are $\beta_{0, i} / (\alpha_{0, i} - 1)$ and $\beta_{0, i}^2 / ((\alpha_{0, i} - 1)^2 (\alpha_{0, i} - 2))$, respectively, this correspond to holding the mean of the variance prior fixed while narrowing its width. In this case, we expect the bound in \cref{thm:bayes_unknown2} to approach that in \cref{thm:bayes_known}, which happens because the term $0.5 / (\kappa_{0, i} (\alpha_{0, i} - 1))$ vanishes. After that, the bounds are similar up a multiplicative factor of $5$.

\textbf{Existing frequentist regret bounds for variance-adaptive Thomson sampling.} \citet{honda14optimality,MV20} proposed variance-adaptive Thompson sampling and bounded its regret. These works differ from us in three aspects. First, the algorithms of \citet{honda14optimality,MV20} are designed for the frequentist setting. Specifically, they have a fixed sufficiently-wide prior, and enjoy a per-instance regret bound under this prior. As an example, the algorithm of \citet{MV20} for $\rho \to \infty$ (Remark 4) is essentially \varts with $\bmu_{0, i} = 0$, $\bkappa_{0, i} = 0$, $\balpha_{0, i} = 0.5$, and $\bbeta_{0, i} = 0.5$. While per-instance regret bounds are strong, the algorithms of \citet{honda14optimality,MV20} can perform poorly when priors are narrower and thus more informative. Truly Bayesian algorithm designs, as proposed in our work, can be analyzed for any informative prior. Second, the analyses of \citet{honda14optimality,MV20} are frequentist. This means that they cannot justify the use of more informative priors and are essentially similar to those of frequentist upper confidence bound algorithms. As an example, in Remark 4 of \citet{MV20}, the authors derive a $O\left(\sum_{i: \mu_i < \mu_{a^*}} \frac{1}{\Delta_i} \log n\right)$ regret bound, where $\Delta_i = \mu_{a^*} - \mu_i$ and $a^*$ is the arm with the highest mean reward $\mu_i$. This bound clearly does not depend on prior parameters, which we incorporate in our bounds. Specifically, our bound in \cref{thm:bayes_unknown2} decreases with lower reward variances and more informative priors on them. Finally, the regret bounds of \citet{honda14optimality,MV20} are asymptotic. We provide strong finite-time guarantees.


%% file: algo_bayes_unknown.tex
\vspace{-3pt}
\subsection{Algorithm \varts}
\label{sec:bayes_unknown_algo}

Similarly to \cref{alg:bayes_known}, the key idea is to maintain a posterior distribution over the unknown mean arm rewards $\bmu$ and act optimistically with respect to samples from it. The challenge is that the reward variances $\bsigma^2$ are also unknown. To overcome this, we rely on the observation that the posterior distribution of $(\mu_i, \sigma_i^{-2})$ is Gaussian-Gamma when the prior is and the rewards are Gaussian. We represent the posterior hierarchically, in an equivalent form (\cref{lem:gauss_gam} in Appendix), as follows. The posterior distribution of the mean reward of arm $i$ in round $t$ is $\cN(\hmu_{t,i},\sigma_{t,i}^2)$, where $\hat \mu_{t, i}$ and $\sigma_{t, i}^2$ are the posterior mean and sampled variance, respectively. The variance is defined as $\sigma_{t,i}^2 = \frac{1}{\kappa_{t,i}\lambda_{t,i}}$, where $\kappa_{t,i} = O(N_t(i))$ and $\lambda_{t,i}$ is a posterior-sampled reward precision of arm $i$ in round $t$. The posterior distribution of $\lambda_{t,i}$ is $\text{Gam}(\alpha_{t,i},\beta_{t,i})$, where $\alpha_{t,i}$ and $\beta_{t,i}$ are its shape and rate parameters, respectively. All posterior parameters are initialized by their prior values $(\mu_{0,i}, \kappa_{0,i}, \alpha_{0,i}, \beta_{0,i})$.

Our algorithm is presented in \cref{alg:bayes_unknown} and we call it \varts, because it adapts to the unknown reward variances of arms. The algorithm works as follows. In round $t$, it first samples the precision of each arm from its posterior, $\lambda_{t,i} \sim \text{Gam}(\alpha_{t,i},\beta_{t,i})$, and then it samples the mean arm reward from its posterior, $\tilde \mu_{t,i} \sim \cN(\hmu_{t,i},\frac{1}{\kappa_{t,i}\lambda_{t,i}})$. After that, the arm with the highest posterior-sampled mean reward is pulled, $A_t:= \arg\max_{i \in [K]}\tilde \mu_{t,i}$. Finally, the algorithm observes a stochastic reward of arm $A_t$, $x_{t,A_t} \sim \cN(\mu_{A_t},\sigma_{A_t}^2)$, and updates its posteriors (lines $9$--$18$ in \cref{alg:bayes_unknown}).

\begin{algorithm}[h]
  \caption{\varts: Gaussian-Gamma TS for unknown reward variances.}
  \label{alg:bayes_unknown}
  \begin{algorithmic}[1]
    \State \textbf{Inputs:} Prior means $\bmu_0$, prior precision $\bkappa_0$, prior shape $\balpha_0$, and prior rate $\bbeta_0$
    \State \textbf{Init:} $\forall i \in [K]: N_1(i):=0, \ \hmu_{1,i}:= \mu_{0,i}, \ \kappa_{1,i}:= \kappa_{0,i}, \ \alpha_{1,i}:= \alpha_{0,i}, \ \beta_{1,i}:= \beta_{0,i}, \ \bar{x}_{1,i}:= 0$. 
    \Statex \vspace{-0.05in}
    \For{$t = 1, \dots, n$} 
    		\State Posterior sampling: $\forall i \in [K]:
      \lambda_{t,i}\sim\text{Gam}(\alpha_{t,i},\beta_{t,i}), \ \tilde \mu_{t,i} \sim \cN(\hmu_{t,i},\frac{1}{\kappa_{t,i}\lambda_{t,i}})$
        \State Pull: $A_t:= \arg\max_{i \in [K]}\tilde \mu_{t,i}$
    		\State Reward feedback: $x_{t,A_t} \sim \cN(\mu_{A_t},\sigma_{A_t}^2)$
    		\State Posterior updates: 
    		\For{$i = 1, \dots, K$}
			\State $N_{t+1}(i):= N_{t}(i) + \I{A_t = i}$    		
    			\State $\beta_{t+1, i}:= \beta_{0,i} + \frac{1}{2}\sum_{s =1}^{t}\I{A_s = i}(x_{s,i} - \bar{x}_{t+1,i})^2 + 
    \frac{\kappa_{0,i} N_{t+1}(i) (\bar{x}_{t+1,i} - \mu_{0,i})^2 }{2(\kappa_{0,i} + N_{t+1}(i))}$ 
  			\State $ \alpha_{t+1, i}
  := \alpha_{0,i} + N_{t+1}(i)/2$
        \State $\kappa_{t+1, i}:= \kappa_{0,i} + N_{t+1}(i)$
  	\State Posterior mean:		
        \State ~~$\hmu_{t+1,i}:= \frac{\kappa_{0,i}\mu_{0,i} + N_{t+1}(i)\bar{x}_{t+1,i}}{\kappa_{0,i}+N_{t+1}(i)}$, where ~$\bar x_{t+1,i}:= \frac{1}{N_{t+1}(i)}\sum_{s=1}^{t} \I{A_{s}=i} x_{s,i}$
    		\EndFor
	\EndFor 
  \end{algorithmic}
\end{algorithm}

%% file: experiments.tex
\section{Experiments}
\label{sec:experiments}

We also study the empirical performance of our proposed algorithms. Since \varts does not assume that the reward variances are known, and thus is more realistic than \cref{alg:bayes_known}, we focus on \varts. We conduct four experiments. First, we evaluate \varts in a Bernoulli bandit, which is a standard bandit benchmark. Second, we experiment with beta reward distributions. Their support is $[0, 1]$, similarly to Bernoulli distributions, but their variances are not fully determined by their means. Since \varts is designed for Gaussian bandits, the first two experiments also evaluate the robustness of \varts to model misspecification. Third, we experiment with a Gaussian bandit. Finally, we vary the number of arms and observe how the performance of \varts scales with problem size.

\subsection{Experimental Setup}

All problems are Bayesian bandits, where the mean arm rewards are sampled from some prior distribution. In a Gaussian bandit, \varts is run with the true $(\bmu_0, \bkappa_0, \balpha_0, \bbeta_0)$. In other problems, the hyper-parameters of \varts are set using the method of moments from samples from the prior. In particular, for a given Bayesian bandit, let $\bar{\mu}_i$ and $v_i$ be the estimated mean and variance of the mean reward of arm $i$ sampled from its prior, respectively. Moreover, let $\bar{\lambda}_i$ and $\nu_i$ be the estimated mean and variance of the precision of reward distribution of arm $i$ sampled from its prior, respectively. Then, based on these statistics, we estimate the unknown hyper-parameters as $\mu_{0, i} = \bar{\mu}_i$, $\beta_{0, i} = \bar{\lambda}_i / \nu_i$, $\alpha_{0, i} = \beta_{0, i} / \bar{\lambda}_i$, and $\kappa_{0, i} = \beta_{0, i} / (\alpha_{0, i} v_i)$ using their empirical estimates \citep{pearson36method}.

We compare \varts to several baselines. \ucb \citep{auer02finitetime} is the most popular algorithm for stochastic $K$-armed bandits with $[0, 1]$ rewards. It does not adapt to the reward variances and is expected to be conservative. We also consider its two variants that adapt to the variances: \ucbtuned \citep{auer02finitetime} and \ucbv \citep{audibert09exploration}. \ucbtuned is a heuristic that performs well in practice. \ucbv uses empirical Bernstein confidence intervals and has theoretical guarantees. We implement both algorithms for $[0, 1]$ rewards. The next two baselines are Thompson sampling algorithms: Bernoulli and Gaussian TS \citep{agrawal13further}. Bernoulli TS has a $\mathrm{Beta}(1, 1)$ prior, as analyzed in \citet{agrawal13further}. When the rewards $Y_{t, i}$ are not binary, we clip them to $[0, 1]$ and then apply Bernoulli rounding: a reward $Y_{t, i} \in [0, 1]$ is replaced with $1$ with probability $Y_{t, i}$ and with $0$ otherwise. Gaussian TS has a $\cN(0, 1)$ prior and unit reward variances, as analyzed in \citet{agrawal13further}. The last two baselines are Thompson sampling with unknown reward variances \citep{honda14optimality,MV20}. We implement Algorithm 1 in \citet{honda14optimality} and call it \htts, and Algorithm 3 in \citet{MV20} for $\rho \to \infty$ and call it \ztts. Note that \ztts is \varts where $\mu_{0, i} = 0$, $\kappa_{0, i} = 0$, $\alpha_{0, i} = 0.5$, and $\beta_{0, i} = 0.5$. The shortcoming of all TS baselines is that they are designed to have frequentist per-instance guarantees. Therefore, their priors are set too conservatively to compete with \varts, which takes the true prior or its estimate as an input. All simulations consider $n = 2\,000$ and are averaged over $1\,000$ randomly initialized runs.

\begin{figure*}[t!]
  \centering
  \includegraphics[width=5.5in]{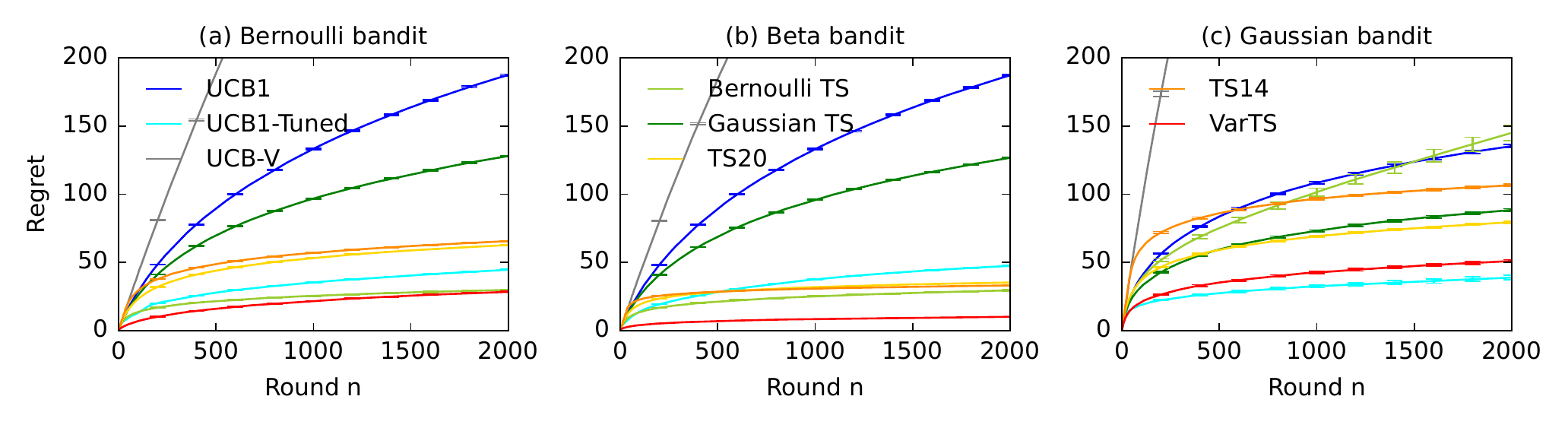}
  \vspace{-0.35in}
  \caption{\varts compared to $7$ baselines. The plots share legends.}
  \label{fig:synthetic}
\end{figure*}

\begin{figure*}[t!]
  \centering
  \includegraphics[width=5.5in]{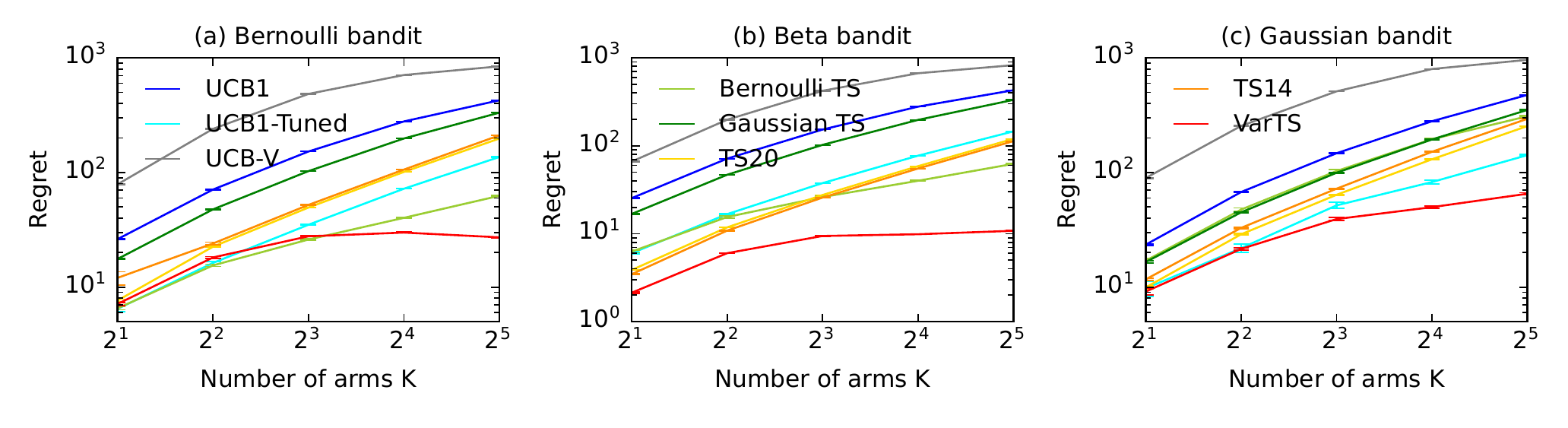}
  \vspace{-0.35in}
  \caption{\varts with $7$ baselines as we vary the number of arms $K$. The plots share legends.}
  \label{fig:scaling}
\end{figure*}

\subsection{Bernoulli Bandit}
\label{sec:bernoulli bandit}

We start with a Bernoulli bandit with $K = 10$ arms. The mean reward of arm $i \in [K]$, $\mu_i$, is sampled i.i.d.\ from prior $\mathrm{Beta}(i, K + 1 - i)$. Since $\E{\mu_i} = i / (K + 1)$ and $\std{\mu_i} \approx 1 / \sqrt{K + 1}$, higher prior means indicate higher $\mu_i$, but it is unlikely that arm $K$ has the highest mean reward.

Our results are reported in \cref{fig:synthetic}a. We observe that \varts and Bernoulli TS have the lowest regret. The latter is not surprising since Bernoulli TS is designed specifically for this problem class. The fact that we match its performance is a testament to adapting to reward variances and using priors. The next three best-performing algorithms (\ucbtuned, \htts, and \ztts) adapt to reward variances but do not use informative priors. The frequentist algorithms with regret bounds (\ucb and \ucbv) have the highest regret because they are too conservative.

\subsection{Beta Bandit}
\label{sec:beta bandit}

The bandit problem in the second experiment is a variant of \cref{sec:bernoulli bandit} where the reward distribution of arm $i$ is $\mathrm{Beta}(s \mu_i, s (1 - \mu_i))$ for $s = 10$. Roughly speaking, this means that the reward variance of arm $i$ is $10$ lower than in \cref{sec:bernoulli bandit}. The rest of the experimental setup is the same as in \cref{sec:bernoulli bandit}.

Our results are reported in \cref{fig:synthetic}b. We observe only two differences from \cref{fig:synthetic}a. First, \varts outperforms Bernoulli TS, because it learns that the arms have $10$ times lower reward variances than in \cref{fig:synthetic}a. Therefore, it can be more aggressive in pulling the optimal arm. Second, both \htts and \ztts outperform \ucbtuned, supposedly due to more principled learning of reward variances.

\subsection{Gaussian Bandit}
\label{sec:gaussian bandit}

The third experiment is with a Gaussian bandit where both the means and variances of rewards are sampled i.i.d.\ from a prior with parameters $\mu_{0, i} = i / (K + 1)$, $\kappa_{0, i} = K$, $\alpha_{0, i} = 4$, and $\beta_{0, i} = 1$. For this setting, $\E{\mu_i} = i / (K + 1)$ and $\std{\mu_i} \approx 1 / \sqrt{K + 1}$. Therefore, higher prior means indicate higher $\mu_i$, but it is unlikely that arm $K$ has the highest mean reward. Moreover, the average reward variance is $0.25$. Therefore, bandit algorithms for $[0, 1]$ rewards are expected to work well.

Our results are reported in \cref{fig:synthetic}c. We observe that \ucbtuned has the lowest regret and \varts performs similarly. This shows the practicality of our design, which is analyzable and comparable to a well-known heuristic without guarantees. All other algorithms have at least $50\%$ higher regret. As before, the frequentist algorithms with regret bounds (\ucb and \ucbv) are overly conservative and among the worst performing baselines.

\subsection{Scalability}
\label{sec:scalability}

We vary the number of arms $K$ and observe how the performance of \varts scales with problem size. This experiment is done in Bernoulli (\cref{sec:bernoulli bandit}), beta (\cref{sec:beta bandit}), and Gaussian (\cref{sec:gaussian bandit}) bandits. Our results in \cref{fig:scaling} shows that the gap between \varts and the baselines increases with $K$. For $K = 32$ and Bernoulli bandit, \varts has at least $3$ times lower regret than any baseline. For $K = 32$ and beta bandit, \varts has at least $5$ times lower regret than any baseline. For $K = 32$ and Gaussian bandit, \varts has at least $2$ times lower regret than any baseline. These gains are driven by adapting to reward variances and using priors, on both the mean reward and its variance.

%% file: conclusions.tex
\section{Conclusions}
\label{sec:concl}

We study the problem of learning to act in a multi-armed Bayesian bandit with Gaussian rewards and heterogeneous reward variances. As a first step, we present a Thompson sampling algorithm for the setting of known reward variances and bound its regret (\cref{thm:bayes_known}). The bound scales as $\sqrt{n\log n} \sqrt{\sum_{i =1}^K\sigma_i^2 \log\left(1 + n\frac{\sigma_{0,i}^{2}}{\sigma_i^{2}}\right)}$. Therefore, it goes to zero as the reward variances $\sigma_i^2$ or the prior variances of the mean arm rewards $\sigma_{0, i}^2$ decrease. Our main contribution is \varts, a variance-adaptive TS algorithm for Gaussian bandits with unknown heterogeneous reward variances. The algorithmic novelty lies in maintaining a joint Gaussian-Gamma posterior for the mean and variance of rewards of each arm. We derive a Bayes regret bound for \varts (\cref{thm:bayes_unknown2}), which scales similarly to the known variance bound. More specifically, it is $\sqrt{n \log n}\sqrt{\sum_{i = 1}^K \frac{\beta_{0,i}}{\alpha_{0,i}-1} \log\left( 1+\frac{n}{\kappa_{0,i}}\right)}$, where $\frac{\beta_{0, i}}{\alpha_{0, i} - 1}$ is a proxy for the reward variance $\sigma_i^2$ and $\kappa_{0, i}^{-1}$ plays the role $\sigma_{0, i}^2/\sigma_i^2$. This bound captures the effect of the prior on learning reward variances and is the first of its kind.

Some interesting directions could be to extend the problem framework to infinite arms, which would require a different set of assumptions on the priors and reward distributions for tractable solutions. Another very practical and general direction could be to incorporate contextualization, where reward variances would vary across items as well as users.

%% file: appendix.tex
\appendix

\onecolumn{

\section*{\centering\Large{Supplementary for \papertitle}}
\vspace*{1cm}

\input{app_known.tex}

\input{app_unknown.tex}


\input{Lemmas.tex}

}

%% file: app_known.tex
\section{Additional proofs for \cref{sec:bayes_known}}
\label{app:known}

\bayesknown*

\begin{proof}[Proof of \cref{thm:bayes_known}]
Recall we denote by $\hat{\bmu}_t \in \realset^K$ the MAP estimate of $\bmu$ at round $t$, and let $\bmu_t \in \realset^K$ be a random posterior sample in round $t$ such that $\bmu_{t,i} \sim \cN(\hat\mu_{t,i},\sigma_{t,i}^2)$ for all $i \in [K]$.  $H_t$ denote the history in round $t$. Note that in posterior sampling, for any vector $\bnu \in \realset^K$, $\condprob{\bmu_t = \bnu}{H_t} = \condprob{\bmu = \bnu}{H_t}$. Let $A^*$ be the optimal arm under the realized reward vector $\bmu$ and $A_t$ be the optimal arm under $\bmu_t$ (and hence pulled by the TS algorithm at round $t$).

We rely on several properties of Gaussian posterior sampling with a diagonal prior covariance matrix. More specifically, following the results from \cref{lem:gaussian posterior update} or \cite{murphy2007conjugate}, the posterior distribution in round $t$ is $\cN(\hat{\bmu}_t, \Sigma_t)$, where $\Sigma_t = \diag{(\sigma_{t, i}^2)_{i = 1}^K}$ is a diagonal covariance matrix with non-zero entries:
\begin{align*}
  \sigma_{t, i}^2
  = \frac{1}{\sigma_{0,i}^{-2} + N_t(i) \sigma_i^{-2}}
  = \frac{\sigma_i^2}{\sigma_i^2 \sigma_{0,i}^{-2} + N_t(i)}\,,
\end{align*}
and $N_t(i)= \sum_{s = 1}^{t-1} \I{A_t = i}$ denotes the number of pulls of arm $i$ up to round $t$. Accordingly, a high-probability confidence interval of arm $i$ in round $t$ is $C_t(i) = \sqrt{2 \sigma_{t, i}^2 \log(1 / \delta)}$, where $\delta > 0$ is the confidence level. Let us define an event:
\begin{align*}
  E_t
  = \set{\forall i \in [K]: \abs{\mu_i -  \hat{\mu}_{t,i}} \leq C_t(i)},
\end{align*}
that implies all confidence intervals in round $t$ hold.

Fix round $t$. Now we bound the regret in round $t$ as follows.  
We start by noting that the regret can be decomposed as:
\begin{align*}
  \E{\mu_{A^*} - \mu_{A_t}}
  & = \E{\condE{\mu_{A^*} - \mu_{A_t}}{H_t}} \\
  & \leq \E{\condE{\mu_{A^*} - {\hat{\mu}_{t,A^*}}}{H_t}} +
  \E{\condE{{{\hat{\mu}_{t,A_t}}} - \mu_{A_t}}{H_t}}
  \\ & = \E{\condE{\mu_{A^*} - {\hat{\mu}_{t,A^*}}}{H_t}}\,.
\end{align*} 

The last equality holds since given $H_t$, clearly $E[\mu_{i}\mid H_t] 
 = \hat{\mu}_{t,i}$ for any $i \in [K]$.

Now let us deal with the remaining term in the decomposition. Fix history $H_t$. Then we introduce event $E_t$ and get
\begin{align*}
  & \condE{\mu_{A^*} - {\hat{\mu}_{t,A^*}} }{H_t}
  \\
  & = \condE{(\mu_{A^*} - {\hat{\mu}_{t,A^*}}) \I{\bar E_t}}{H_t}
  + \condE{(\mu_{A^*} - {\hat{\mu}_{t,A^*}}) \I{E_t}}{H_t}
  \\
  & \leq \condE{(\mu_{A^*} - {\hat{\mu}_{t,A^*}}) \I{\bar{E}_t}}{H_t} + \condE{C_t(A_t)}{H_t}\,,
\end{align*}
where the inequality follows from the observation that

\begin{align*}
  \condE{(\mu_{A^*} - {\hat{\mu}_{t,A^*}}) \I{E_t}}{H_t}
  \leq \condE{C_t(A^*)}{H_t}\,,
\end{align*}
and further $A_t \mid H_t$ and $A^* \mid H_t$ have the same distributions given $H_t$. 
Now note since $\bmu - \hat{\bmu}_t \mid H_t \sim \cN(\mathbf{0}, \Sigma_t)$, we further have
\begin{align}
  \condE{(\mu_{A^*} - {\hat{\mu}_{t,A^*}}) \I{\bar{E}_t}}{H_t}
  & \leq \sum_{i = 1}^K \frac{1}{\sqrt{2 \pi \sigma_{t, i}^2}}
  \int_{x = C_t(i)}^\infty x
  \exp\left[- \frac{x^2}{2 \sigma_{t, i}^2}\right] \dif x
  \nonumber \\
  & = \sum_{i = 1}^K - \sqrt{\frac{\sigma_{t, i}^2}{2 \pi}}
  \int_{x = C_t(i)}^\infty \frac{\partial}{\partial x}
  \left(\exp\left[- \frac{x^2}{2 \sigma_{t, i}^2}\right]\right) \dif x
  \nonumber \\
  & = \sum_{i = 1}^K \sqrt{\frac{\sigma_{t, i}^2}{2 \pi}} \delta
  \leq \sqrt{\frac{\sigma_{0,i}^2}{2 \pi}}  \delta\,.
  \label{eq:bayes regret scale}
\end{align}

Now we chain all inequalities for the regret in round $t$ and get
\begin{align*}
  \E{\mu_{A^*} - \mu_{A_t}}
  \leq \E{C_t(A_t)} + \sum_{i = 1}^K\sqrt{\frac{2 \sigma_{0,i}^2}{\pi}}  \delta\,.
\end{align*}
Therefore, the $n$-round Bayes regret is bounded as
\begin{align*}
  \E{\sum_{t = 1}^n \mu_{A^*} - \mu_{A_t}}
  \leq \E{\sum_{t = 1}^n C_t(A_t)} +
  \sum_{i = 1}^K\sqrt{\frac{2 \sigma_{0,i}^2}{\pi}} n \delta\,.
\end{align*}

The last part is to bound $\E{\sum_{t = 1}^n C_t(A_t)}$ from above. Since the confidence interval $C_t(i)$ decreases with each pull of arm $i$, $\sum_{t = 1}^n C_t(A_t)$ is bounded for any $\mu$ by pulling arms in a round robin \citep{russo14learning}, which yields

\begin{align*}
  \E{\sum_{t = 1}^n C_t(A_t)}
  & = \E{\sum_{t = 1}^n \sum_{i = 1}^K \I{A_t = i} C_t(i)} 
  = \E{\sum_{i = 1}^K \bigg[ \sum_{t = 1}^n \I{A_t = i} C_t(i) \bigg]}
  \\
  &\leq \E{\sum_{i = 1}^K \bigg[ \sqrt{N_t(i)} \sqrt{\sum_{t = 1}^n C_t^2(i)\I{A_t = i} }\bigg]}
  \\
  &=  \E{\sum_{i = 1}^K \bigg[ \sqrt{N_t(i)}  \sqrt{\sum_{t = 1}^n \frac{2\sigma_i^2}{\sigma_i^2 \sigma_{0,i}^{-2} + N_t(i)} \log(1 / \delta)\I{A_t = i}}\bigg]}
  \\
  &=  \E{\sum_{i = 1}^K \bigg[ \sqrt{2N_t(i)\sigma_i^2}  \sqrt{\sum_{s = 1}^{N_n(i)} \frac{1}{\sigma_i^2 \sigma_{0,i}^{-2} + s} \log(1 / \delta) + \frac{\sigma_{0,i}^2}{\sigma_i^2}\log(1/\delta)}\bigg]}
  \\
  &\leq \E{\sum_{i = 1}^K \bigg[ \sqrt{2N_t(i)\sigma_i^2}  \sqrt{\sum_{s = 1}^{n} \frac{1}{\sigma_i^2 \sigma_{0,i}^{-2} + s} \log(1 / \delta) + \frac{\sigma_{0,i}^2}{\sigma_i^2}\log(1/\delta)}\bigg]}
  \\	  
  & \overset{(a)}{\leq} \E{\sum_{i = 1}^K \bigg[ \sqrt{2N_t(i)\sigma_i^2}  \sqrt{\Big(\log(1 + n\sigma_{0,i}^{2}\sigma_i^{-2}) + \sigma_{0,i}^{2}\sigma_i^{-2}\Big)\log(1/\delta)}\bigg]}
  \\
  &= \E{\sum_{i = 1}^K \bigg[ \sqrt{2N_t(i)}  \sqrt{\sigma_i^2\Big(\log(1 + n\sigma_{0,i}^{2}\sigma_i^{-2}) + \sigma_{0,i}^{2}\sigma_i^{-2}\Big)\log(1/\delta)}\bigg]}
  \\
  & \leq  \sqrt{2n} \sqrt{\sum_{i =1}^K\sigma_i^2\Big(\log(1 + n\sigma_{0,i}^{2}\sigma_i^{-2}) + \sigma_{0,i}^{2}\sigma_i^{-2}\Big)\log(1/\delta)}\\
  & \leq  \sqrt{2n} \sqrt{\sum_{i =1}^K\sigma_i^2\log(1 + n\sigma_{0,i}^{2}\sigma_i^{-2})\log(1/\delta)} + \sqrt{2n\sum_{i = 1}^K\sigma_{0,i}^{2}\log(1/\delta)}
  \,,
\end{align*}
where the first and the last inequality follows due to Cauchy-Schwarz, and Inequality $(a)$ is due to \cref{lem:reciprocal sum}. 
Chaining all inequalities completes the proof.
\end{proof}

%% file: app_unknown.tex
\section{Additional proofs for \cref{sec:bayes_unknown2}}
\label{app:unknown}

\bayesunknown*

\begin{proof}[Proof of \cref{thm:bayes_unknown2}]
We start by noting that the posterior updates (from results of \cite{murphy2007conjugate}): 

\begin{align*}
    & \mu_{t,i} = \frac{\kappa_{0,i}\mu_{0,i} + N_t(i)\bar{x}_{t,i}}{\kappa_{0,i}+N_t(i)}
    \\
    & \kappa_{t,i} = \kappa_{0,i} + N_t(i)
    \\
    & \alpha_{t,i} = \alpha_{0,i} + N_t(i)/2
    \\
    & \beta_{t,i} = \beta_{0,i} 
    +
    \frac{1}{2}\sum_{t \in [T]}\I{A_t = i}(x_{t,i} - \bar{x}_{t,i})^2
    +
    \frac{\kappa_{0,i} N_t(i) (\bar{x}_{t,i} - \mu_{0,i})^2 }{2(\kappa_{0,i} + N_t(i))},
\end{align*}
where $N_t(i)= \sum_{s = 1}^{t-1} \I{A_t = i}$ denotes the number of pulls of arm $i$ up to round $t$, and $\bar x_{t,i}= \frac{1}{N_t(i)}\sum_{s = 1}^{t-1} \I{A_t = i}x_{t,i}$ being the averaged empirical mean reward of arm-$i$ at round $t$. 
Thus we have: 

\[
\mu_{i} - \hat \mu_{t,i}  \mid (\kappa_t,\lambda_t) \sim \cN(0 ,(\kappa_{t,i} \lambda_{t,i})^{-1},
\]
where $\lambda_{t,i} \sim \text{Gamma}(\alpha_{t,i},\beta_{t,i})$. 

Same as proof of \cref{thm:bayes_known}, we can break the regret in round $t$ as before:
\begin{align}
\label{eq:term5}
  \E{\mu_{A^*} - \mu_{A_t}}
  & = \E{\condE{\mu_{A^*} - \mu_{A_t}}{H_t}} \nonumber \\
  & = \E{\condE{\mu_{A^*} - {\hat{\mu}_{t,A^*}}}{H_t}} +
  \E{\condE{{{\hat{\mu}_{t,A_t}}}- \mu_{A_t}}{H_t}}\,,
  \nonumber \\
  & = \E{\condE{\mu_{A^*} - {\hat{\mu}_{t,A^*}}}{H_t}} +
  \mathbb E\Bigsn{\mathbb E_{\bsigma}\bigsn{\mathbb E [\hat{\mu}_{t,A_t}- \mu_{A_t} \mid \bsigma ] \mid H_t} }\,, \nonumber
  \\
  & = \E{\condE{\mu_{A^*} - {\hat{\mu}_{t,A^*}}}{H_t}} 
\end{align}

where $\sigma_{t,i}^2$ is defined as $\sigma_{t,i}^2= \frac{1}{{\kappa_{t,i}\lambda_{t,i}}}$ which is a `sampled posterior variance', (where recall $\lambda_{t,i} \sim \text{Gam}(\alpha_{t,i},\beta_{t,i})$). The last equality holds since given $H_t$ and $\bsigma$,  $E[\mu_{i} \mid H_t] 
 = \hat{\mu}_{t,i}$ for any $i \in [K]$ by definition of the posterior update. 

Now given $H_t$ and $\bsigma_t$, let us define the high-probability confidence interval of each arm $i$ as $C_t(i) = \sqrt{2 \sigma_{t, i}^2 \log(1 / \delta)}$, and the ``good event"
$
  E_t
  = \set{\forall i \in [K]: \abs{\mu_i -  \hat{\mu}_{t,i}} \leq C_t(i)}
$,
same as what we introduced in the proof of \cref{thm:bayes_known}. 
%
Further note, given $H_t$ and $\sigma_{t,i}^2$ (or equivalently $\lambda_{t,i})$),  $\mu_i - \hat{\mu}_{t,i} \mid \sigma_{t,i}^2, H_t \sim \cN(\mathbf{0}, \sigma_{t,i}^2)$, 
since given $H_t$ and $\sigma_t$, $\mu_i$ has the posterior $\mu_i \sim \cN(\hat \mu_{t,i}, \sigma_{t,i}^{2})$, where recall we defined $\sigma_{t,i}^2 = \frac{1}{\kappa_{t,i}\lambda_{t,i}}$ and $\lambda_{t,i} = \frac{1}{\sigma_{t,i}^2} \sim \text{Gam}(\alpha_{t,i},\beta_{t,i})$.
Thus following same analysis as in \eqref{eq:bayes regret scale}, we get: 
%
\begin{align}
  \condE{(\mu_{A^*} - {\hat{\mu}_{t,A^*}}) \I{\bar{E}_t}}{\bsigma_t, H_t}
  \leq \sum_{i = 1}^K \sqrt{\frac{\sigma_{t,i}^2}{2 \pi}} \delta.
  \label{eq:bayes_regret_scale22}
\end{align} 

Further taking expectation over $H_n, \bsigma, \bmu$, and summing over $t = 1,2, \ldots n$ we get:
\begin{align*}
  &\E{\sum_{t = 1}^n(\mu_{A^*} - {\hat{\mu}_{t,A^*}}) \I{\bar{E}_t}}
  \\	
  & \leq 
  \delta(2\pi)^{-1/2}\mathbb{E}_{H_n, \bsigma, \bmu}\biggsn{\sum_{t = 1}^n \sum_{i = 1}^K \mathbb{E}_{\bsigma_{t}}[\sigma_{t,i} \mid \bmu,\bsigma, H_n]}
  \\
  & = 
    \delta(2\pi)^{-1/2}\mathbb{E}_{H_n, \bsigma, \bmu}\biggsn{\sum_{t = 1}^n \sum_{i = 1}^K \mathbb E_{\lambda_{t,i} \sim \text{Gam}(\alpha_{t,i},\beta_{t,i})}\Bigsn{ \sqrt{\frac{1}{2 \kappa_{t,i}\lambda_{t,i}}} \mid \bmu,\bsigma, H_n } }
  \\
  & \leq  
  \delta(2\pi)^{-1/2}\mathbb{E}_{H_n, \bsigma, \bmu}\biggsn{\sum_{t = 1}^n \sum_{i = 1}^K \biggsn {\sqrt{\mathbb E_{\lambda_{t,i}}\bigsn{ \frac{1}{2 \kappa_{t,i}\lambda_{t,i}} \mid \bmu,\bsigma, H_n  } }}  }
  \\
    & \overset{(1)}{=}
  \delta(2\pi)^{-1/2}\mathbb{E}_{H_n, \bsigma, \bmu}\biggsn{ \sum_{i = 1}^K \sum_{t = 1}^n \biggsn {\sqrt{ \bigsn{ \frac{\beta_{t,i}}{2 \kappa_{t,i}(\alpha_{t,i}-1)}} }} \mid \bmu,\bsigma, H_n  }
  \\
  & \overset{(2)}{\leq}  
  \delta(2\pi)^{-1/2}\biggsn{ \biggsn {\sqrt{ nK   \mathbb{E}_{H_n, \bsigma, \bmu}\bigsn{ \sum_{t = 1}^n \sum_{i = 1}^K  \frac{\beta_{t,i}}{2 \kappa_{t,i}(\alpha_{t,i}-1)} \mid \bmu,\bsigma, H_n  } }}  } =
  \\
  &\hspace{-0.1in}   
  \frac{\delta}{\sqrt{2\pi}}\biggsn{ \biggsn {\sqrt{ nK   \mathbb{E}_{H_n, \bsigma, \bmu}\bigsn{ \sum_{i = 1}^K \sum_{\tau = 1}^n  \frac{2\beta_{0,i} 
    +
    \sum_{t \in [\tau]}\I{A_t = i}(x_\tau(i) - \bar{x}_{\tau,i})^2
    +
    \kappa_{0,i} [(\bar{x}_{\tau,i} - \mu_{i})^2 + (\mu_{i} - \mu_{0,i})^2  ]}{2 \kappa_{t,i}(\alpha_{0,i} + N_{\tau}(i)/2 -1)}  \mid \bmu,\bsigma, H_n  } }}},
\end{align*}
where $(1)$ holds since since $1/\lambda_{t,i}$ follows Inverse-Gamma$(\alpha_{t,i},\beta_{t,i})$ and thus we can apply \cref{lem:inv_gam} (note we assumed $\alpha_{0,i}> 1$, which implies $\alpha_{t,i}> 1, ~\forall t \in [T]$).
$(2)$ applies by successive application of Cauchy-Schwarz and pushing the expectation under square-root by Jensen's inequality.

Further, applying the posterior update forms from \cref{alg:bayes_unknown}, we get: 

\begin{align*}
  &\frac{\sqrt{2\pi}}{\delta}\E{\sum_{t = 1}^n(\mu_{A^*} - {\hat{\mu}_{t,A^*}}) \I{\bar{E}_t}}
  \\	
  & \leq   
  \biggsn{ \biggsn {\sqrt{ nK \mathbb{E}_{H_n, \bsigma, \bmu}\bigsn{   \sum_{i = 1}^K \sum_{\tau = 1}^n  \frac{2\beta_{0,i} 
    +
    \sum_{t \in [\tau]}\I{A_t = i}(x_\tau(i) - \bar{x}_{\tau,i})^2
    +
    \kappa_{0,i} [(\bar{x}_{\tau,i} - \mu_{i})^2 + (\mu_{i} - \mu_{0,i})^2  ]}{2 (\kappa_{0,i}+N_\tau(i))(\alpha_{0,i} + N_{\tau}(i)/2 -1)} \mid \bmu,\bsigma, H_n } }}   }
    \\
    & =  
    \biggsn{ \biggsn {\sqrt{ nK \mathbb{E}_{H_n, \bsigma, \bmu}\bigsn{ \sum_{i = 1}^K \mathbb E_{N_n(i)}\Bigsn{\sum_{s = 0}^{N_n(i)}  \frac{2\beta_{0,i} 
    +
    \sum_{t = 1}^s(x_s(i) - \bar x_{s,i} )^2
    +
    \kappa_{0,i} [(\bar{x}_{s,i} - \mu_{i})^2 + (\mu_{i} - \mu_{0,i})^2  ]}{ (\kappa_{0,i}+s)(2\alpha_{0,i} + s -2)}} \mid \bmu,\bsigma, H_n  }}}  },
\end{align*}
where we we slightly abused the notation above to denote $\bar x_{s,i}= \frac{1}{s}\sum_{t' = 1}^s x_{t',i}$ for all $s > 1$. However, further to get rid of the randomness of $N_n(i)$, we can further upper bound the above expression as:
\begin{align}
\label{eq:long0}
  &\frac{\sqrt{2\pi}}{\delta}\E{\sum_{t = 1}^n(\mu_{A^*} - {\hat{\mu}_{t,A^*}}) \I{\bar{E}_t}}
  \\	
  & \leq   \nonumber
    \biggsn{ \biggsn {\sqrt{ nK   \mathbb{E}_{H_n, \bsigma, \bmu}\bigsn{ \sum_{i = 1}^K \sum_{s = 0}^{n}  \frac{2\beta_{0,i} 
    +
    \sum_{t = 1}^s(x_s(i) - \bar x_{s,i} )^2
    +
    \kappa_{0,i} [(\bar{x}_{s,i} - \mu_{i})^2 + (\mu_{i} - \mu_{0,i})^2  ]}{(\kappa_{0,i}+s)(2\alpha_{0,i} + s -2)}\mid \bmu,\bsigma, H_n }}}   },
\end{align}
which after applying \cref{lem:var_bdd} we get:

\begin{align}
\label{eq:long1}
  &\E{\sum_{t = 1}^n(\mu_{A^*} - {\hat{\mu}_{t,A^*}}) \I{\bar{E}_t} }
  \\	
  & \leq   \nonumber
    \frac{\delta \sqrt{nK}}{\sqrt{2\pi}}\times
    \\
    &\biggsn{ \biggsn {\sqrt{ \mathbb{E}_{H_n, \bsigma, \bmu} \bigsn{ \underbrace{\sum_{i = 1}^K \frac{2 \beta_{0,i} + \kappa_{0,i}(\mu_i-\mu_{0,i})^2}{2\kappa_{0,i}(\alpha_{0,1}-1)}}_{(A)} + \underbrace{ \sum_{i = 1}^K \sum_{s = 1}^{n}  \frac{ 2\beta_{0,i} 
    +
    \sigma_i^2 (s-1)
    +
    \kappa_{0,i} [\frac{\sigma_{i}^2}{s} + (\mu_{i} - \mu_{0,i})^2  ]  }{(\kappa_{0,i} + s)(2\alpha_{0,i} + s - 2)}}_{(B)} \mid \bmu,\bsigma, H_n }}}   },
\end{align}

For the term $(A)$, since $\mu \sim \cN(\mu_{0,i},\frac{1}{\kappa_{0,i}\lambda_{i}})$, where $\lambda_{i} \sim \text{Gam}(\alpha_{0,i},\beta_{0,i})$, taking conditional expectation over $\mu_i \mid \lambda_{0,i}$ and further over $\lambda_{0,i}$ we get:

\[
\mathbb E_{\lambda_{i}}\biggsn{ \mathbb E_{\mu_{i}}\Bigsn{ \bigsn{\sum_{i = 1}^K \frac{2 \beta_{0,i} + \kappa_{0,i}(\mu_i-\mu_{0,i})^2}{2\kappa_{0,i}(\alpha_{0,1}-1)}} \mid \lambda_i } }
= 
\sum_{i = 1}^K \frac{2 \beta_{0,i} + \mathbb E_{\lambda_{i}}\biggsn{ \lambda_i^{-1}}}{2\kappa_{0,i}(\alpha_{0,1}-1)}
=   \sum_{i = 1}^K \frac{2 \beta_{0,i} + \frac{\beta_{0,i}}{\alpha_{0,i}-1}}{2\kappa_{0,i}(\alpha_{0,1}-1)},
\]
where note the last equality holds since since $1/\lambda_{i}$ follows Inverse-Gamma$(\alpha_{0,i},\beta_{0,i})$ and thus we can apply \cref{lem:inv_gam}, given we assumed $\alpha_{0,i}> 1$ by assumption.
  
To control term $(B)$, 
\begin{align*}
& \mathbb E_{H_n,\bsigma,\bmu} \biggsn{ \sum_{i = 1}^K \sum_{s = 1}^{n}  \frac{ 2\beta_{0,i} 
    +
    \sigma_i^2 (s-1)
    +
    \kappa_{0,i} [\frac{\sigma_{i}^2}{s} + (\mu_{i} - \mu_{0,i})^2  ]  }{(\kappa_{0,i} + s)(2\alpha_{0,i} + s - 2)}}
    \\
    & \leq
    \mathbb E_{H_n,\bsigma,\bmu} \biggsn{ \sum_{i = 1}^K \sum_{s = 1}^{n}  \frac{ 2\beta_{0,i} 
    +
    \sigma_i^2 s
    +
    \kappa_{0,i} [\frac{\sigma_{i}^2}{s} + (\mu_{i} - \mu_{0,i})^2  ]  }{(\kappa_{0,i} + s)(2\alpha_{0,i} + s - 2)}}
    \\
      & \overset{(1)}{\leq}
      \sum_{i = 1}^K \frac{\beta_{0,i}}{(\kappa_{0,i}+1)(\alpha_{0,i}-0.5)} + \sum_{i = 1}^K \sum_{s = 1}^{n} \frac{\beta_{0,i}/(\alpha_{0,i}-1)}{\kappa_{0,i}+s}
      + 
    \mathbb E_{\bsigma} \biggsn{ \sum_{i = 1}^K \sum_{s = 1}^{n} \frac{\sigma_i^2}{\kappa_{0,i} + s} }   
    \\
    & \hspace{2in} + \mathbb E_{\bsigma} \biggsn{ \sum_{i = 1}^K \sum_{s = 1}^{n} \frac{
     \sigma_{i}^2  }{s^2}} 
    +
    \mathbb E_{\bsigma} \biggsn{ \sum_{i = 1}^K \sum_{s = 1}^{n}\frac{
     \mathbb E_{\bmu}[(\mu_{i} - \mu_{0,i})^2 \mid \sigma  ]  }{2(\alpha_{0,i}-1) + s}} 
    \\
    & \overset{(2)}{\leq}
      \sum_{i = 1}^K \frac{\beta_{0,i}}{(\kappa_{0,i}+1)(\alpha_{0,i}-0.5)} + \sum_{i = 1}^K \frac{\beta_{0,i}}{\alpha_{0,i}-1}\log\Big(1 + \frac{n}{\kappa_{0,i}}\Big)
      + 
    \biggsn{ \sum_{i = 1}^K  \frac{\beta_{0,i}}{\alpha_{0,i}-1}\log\Big(1 + \frac{n}{\kappa_{0,i}}\Big) }   
    \\
    & \hspace{2in} + \biggsn{ \sum_{i = 1}^K  \frac{
    \beta_{0,i}  }{(\alpha_{0,i}-1)}\frac{\pi^2}{6}} 
    +
    \mathbb E_{\bsigma} \biggsn{ \sum_{i = 1}^K 
    \kappa_{0,i}\frac{\sigma_i^2}{\kappa_{0,i}} \log\Big(1 + \frac{n}{\kappa_{0,i}}\Big)   } 
    \\
    & \leq 
    \sum_{i = 1}^K \frac{\beta_{0,i}}{\kappa_{0,i}(\alpha_{0,i}-1)}
    +
   \sum_{i = 1}^K \frac{(3+\frac{\pi^2}{6})\beta_{0,i}}{(\alpha_{0,i}-1)}\log\Big(1 + \frac{n}{\kappa_{0,i}}\Big)
   \\
   & \leq 
   \sum_{i = 1}^K \frac{\beta_{0,i}}{\kappa_{0,i}(\alpha_{0,i}-1)}
    +
   \sum_{i = 1}^K \frac{5\beta_{0,i}}{(\alpha_{0,i}-1)}\log\Big(1 + \frac{n}{\kappa_{0,i}}\Big)
\end{align*} 
where Inequality $(1)$ several times uses the fact that $2(\alpha_{0,i} - 1)>0$, and Inequality $(2)$ comes from repeated application of \cref{lem:reciprocal sum} and the fact that $\sum_{x = 1}^\infty 1/x^2 \leq \frac{\pi^2}{6}$.

Then combining the above two derived upper bounds of term (A) and (B) to \eqref{eq:long1}, we further have: 

\begin{align}
\label{eq:long2}
  &\E{\sum_{t = 1}^n(\mu_{A^*} - {\hat{\mu}_{t,A^*}}) \I{\bar{E}_t}} \nonumber
  \\	
  & \leq   
    \frac{\delta \sqrt{nK}}{\sqrt{2\pi}} \sqrt{\sum_{i = 1}^K \frac{4 \beta_{0,i} + \frac{\beta_{0,i}}{\alpha_{0,i}-1}}{2\kappa_{0,i}(\alpha_{0,1}-1)}
+ 
\sum_{i = 1}^K \frac{5\beta_{0,i}}{(\alpha_{0,i}-1)}\log\Big(1 + \frac{n}{\kappa_{0,i}}\Big)}
    \,.
\end{align}

Coming back to the original Bayesian regret expression, note that the regret expression from \eqref{eq:term5} can be actually decomposed based on $E_t$ and $\bar E_t$ as:
\begin{align}
\label{eq:term2}
  &\condE{{\hat{\mu}_{t,A^*}}  - \mu_{A^*}}{H_t} \nonumber
  \\  
  &\leq \condE{({\hat{\mu}_{t,A^*}} - \mu_{A^*}) \I{E_t}}{H_t} +
  \condE{({\hat{\mu}_{t,A^*}} - \mu_{A^*}) \I{\bar{E}_t}}{H_t} \nonumber
  \\
  & \leq \condE{C_t(A^*)}{H_t} \nonumber
  +\condE{({\hat{\mu}_{t,A^*}} - \mu_{A^*}) \I{\bar{E}_t}}{H_t}
  \\
  & = \condE{C_t(A_t)}{H_t} 
  +\condE{({\hat{\mu}_{t,A^*}} - \mu_{A^*}) \I{\bar{E}_t}}{H_t}
  \,,
\end{align}
where the last equality follows from the fact that given $H_t$, $A_t \mid H_t$ and $A^* \mid H_t$ have the same distributions, as both $\tilde \mu_{t,i}$ and $\mu_{t,i}$ can be seen as independent posterior samples from a NG$(\hat \mu_{t,i},\kappa_{t,i},\alpha_{t,i},\beta_{t,i})$ distribution marginalized over $\lambda$; more precisely for any realization of $\tilde \mu_{t,i}$ and $\mu_i$:
\[
P(\tilde \mu_{t,i}) = \int_{\lambda} P\big( (\tilde \mu_{t,i}, \lambda_{t,i} \sim \text{NG}(\hat \mu_{t,i},\kappa_{t,i},\alpha_{t,i},\beta_{t,i}) ) \big)d\lambda  ~~~\text{and}
\]
\[
P(\mu_{i}) = \int_{\lambda} P\big( (\tilde \mu_{t,i}, \lambda_{t,i} \sim \text{NG}(\hat \mu_{t,i},\kappa_{t,i},\alpha_{t,i},\beta_{t,i}) ) \big)d\lambda,
\]
they follow the same distribution.

Bounding the second term in \eqref{eq:term2} follows as in \eqref{eq:long2}. 
The only remaining task is to bound the first term in \eqref{eq:term2}, which can be upper bounded as follows:

\begin{align*}
  &\E{\sum_{t = 1}^n C_t(A_t)}
  = \E{\sum_{t = 1}^n \sum_{i = 1}^K \I{A_t = i} C_t(i)} 
  = \E{\sum_{i = 1}^K \bigg[ \sum_{t = 1}^n \I{A_t = i} C_t(i) \bigg]}
  \\
  &\overset{(1)}{\leq} \E{\sum_{i = 1}^K \bigg[ \sqrt{N_n(i)} \sqrt{\sum_{t = 1}^n C_t^2(i)\I{A_t = i} }\bigg]}
  \\
  &\overset{(2)}{\leq} \E{ \bigg[ \sqrt{\sum_{i = 1}^K N_n(i)} \sqrt{\sum_{i = 1}^K \sum_{t = 1}^n C_t^2(i)\I{A_t = i} }\bigg]}
  \\
  &\overset{}{=}  \bigg[ \sqrt{n} \sqrt{\E{ \sum_{i = 1}^K \sum_{t = 1}^n {2 \sigma_{t, i}^2 \log(1 / \delta)}\I{A_t = i} }\bigg]}
  \\
  &\overset{}{=}  \bigg[ \sqrt{n} \sqrt{\E{ \sum_{i = 1}^K \sum_{s = 0}^{N_n(i)} \bign{2 \sigma_{t, i}^2 \log(1 / \delta)}\I{A_t = i} }\bigg]}
  \\
  & \overset{}{\leq}  \bigg[ \sqrt{n} \sqrt{\E{ \sum_{i = 1}^K \sum_{s = 0}^{n} \bign{2 \sigma_{t, i}^2 \log(1 / \delta)}\I{A_t = i} }\bigg]}
  \\
  & \overset{(4)}{\leq} \bigg[ \sqrt{n\log \frac{1}{\delta}} \sqrt{{ \sum_{i = 1}^K \sum_{s = 0}^n   \mathbb{E}_{H_n, \bsigma, \bmu}\bigsn{ \sum_{i = 1}^K \sum_{s = 0}^{n}  \frac{2\beta_{0,i} 
    +
    \sum_{t = 1}^s(x_s(i) - \bar x_{s,i} )^2
    +
    \kappa_{0,i} [(\bar{x}_{s,i} - \mu_{i})^2 + (\mu_{i} - \mu_{0,i})^2  ]}{(\kappa_{0,i}+s)(2\alpha_{0,i} + s -2)}\mid \bmu,\bsigma, H_n } }\bigg]}
  \\
  & \leq \sqrt{n\log \frac{1}{\delta}}\sqrt{\sum_{i = 1}^K \sum_{i = 1}^K \frac{4 \beta_{0,i} + \frac{\beta_{0,i}}{\alpha_{0,i}-1}}{2\kappa_{0,i}(\alpha_{0,1}-1)}
+ 
\sum_{i = 1}^K \frac{5\beta_{0,i}}{(\alpha_{0,i}-1)}\log\Big(1 + \frac{n}{\kappa_{0,i}}\Big)}
 \,,
\end{align*}
where $(1)$ and $(2)$ uses Cauchy-Schwarz, $(2)$ applies Jensen's inequality, the Inequality $(4)$ above follows similar to the derivation of \eqref{eq:long0}, and the inequality follows from a similar derivation as shown for \eqref{eq:long2} from \eqref{eq:long0}.

Finally, chaining the above inequalities on the summation of the instantaneous regret over $n$ rounds (see \eqref{eq:term2})

\begin{align*}
\label{eq:term2}
  &\mathbb{E}_{H_n,\bmu,\bsigma}\biggsn{\sum_{t = 1}^n\condE{{\hat{\mu}_{t,A^*}}  - \mu_{A^*}}{H_t}} \nonumber
  \\  
  &\leq
  \biggn{\frac{\delta \sqrt{nK}}{\sqrt{2\pi}} + \sqrt{n\log \frac{1}{\delta}}}\sqrt{\sum_{i = 1}^K \frac{4 \beta_{0,i} + \frac{\beta_{0,i}}{\alpha_{0,i}-1}}{2\kappa_{0,i}(\alpha_{0,1}-1)}
+ 
\sum_{i = 1}^K \frac{5\beta_{0,i}}{(\alpha_{0,i}-1)}\log\Big(1 + \frac{n}{\kappa_{0,i}}\Big)}  
  \,,
\end{align*}
which proves the claim, noting $C= \sum_{i = 1}^K \frac{4 \beta_{0,i} + \frac{\beta_{0,i}}{\alpha_{0,i}-1}}{2\kappa_{0,i}(\alpha_{0,1}-1)}
+ 
\sum_{i = 1}^K \frac{5\beta_{0,i}}{(\alpha_{0,i}-1)}\log\Big(1 + \frac{n}{\kappa_{0,i}}\Big)$. 
\end{proof}

\subsection{Additional Claims used in the Proof of \cref{thm:bayes_unknown2}} 

\begin{restatable}[\cite{murphy2007conjugate}]{lem}{gaussgam}
\label{lem:gauss_gam}
The joint distribution of $(\mu,\lambda) \sim NG(\mu_0,\kappa_0,\alpha_0,\beta_0)$ is equivalent to the consecutive draw of $\lambda \sim \text{Gam}(\alpha_0,\beta_0)$ followed by $\mu \sim \cN(\mu_0,(\kappa_0\lambda)^{-1})$, $(\mu_0,\kappa_0,\beta_0,\alpha_0)$ being the paraeters of the Gaussian-Gamma distribution. In other words, 
\[
NG(\mu,\lambda \mid \mu_0,\kappa_0,\alpha_0,\beta_0) = \cN(\mu \mid \mu_0,(\kappa_0\lambda)^{-1})\text{Gam}(\gamma \mid \alpha_0,\beta_0) 
\] 
\end{restatable}

\begin{restatable}[\cite{cook08}]{lem}{invgam}
\label{lem:inv_gam}
Suppose $X$ follows a Gamma$(\alpha,\beta)$ distribution with shape parameter $\alpha>0$ and rate parameter $\beta > 0$, then $1/X$ follows an Inverse-Gamma$(\alpha,1/\beta)$ distribution. 
\\
Moreover if $\alpha>1$, $E[1/X] = \frac{\beta}{\alpha-1}$. 
\end{restatable}

\begin{restatable}[\cite{cook08}]{lem}{avgchi}
\label{lem:avg_chi}
Let ${\displaystyle X_{1},...,X_{n}}$ are i.i.d. ${\displaystyle N(\mu ,\sigma ^{2})}$ random variables, (a). then 
\[
{\displaystyle \sum _{i=1}^{n}\frac{(X_{i}-{\bar {X}_n})^{2}}{\sigma ^{2}}\sim \chi _{n-1}^{2}},
\]
where ${\displaystyle {\bar {X}_n}={\frac {1}{n}}\sum _{i=1}^{n}X_{i}}$ is the sample mean of $n$ random draws, and $\chi_{d}^2$ is the chi-squared random variable with $d$-degrees of freedom (for any $d \in \N_+$). 
\\
(b). Further if $X \sim \chi_d^2$, then $E[X] = d$.
\end{restatable}

\begin{restatable}[]{lem}{varbdd}
\label{lem:var_bdd}
For any $i \in [K]$, and round $\tau \in [T]$,
\begin{align*}
  E_{H_\tau,\bmu,\bsigma}&\biggsn{\frac{\beta_{\tau,i}}{ (\kappa_{\tau,i})(\alpha_{\tau,i} -1)} \mid N_\tau(i)} 
    \leq  
        E_{H_\tau,\bmu,\bsigma}\biggsn{
    \frac{\Bigsn{2\beta_{0,i} 
    +
    \sigma_i^2 (N_{\tau}(i)-1)
    +
    \kappa_{0,i} [\frac{\sigma_{i}^2}{N_{t,i}} + (\mu_{i} - \mu_{0,i})^2  ]  }}
    {\kappa_{\tau,i}(2\alpha_{0,i} + 2N_{\tau}(i) - 2)} \mid N_\tau(i)}
\end{align*}
\end{restatable}

\begin{proof}
We start by noting that for any $i \in [K]$, 
\begin{align*}
 \mathbb{E}_{H_\tau,\bsigma}&\Bigsn{ \mathbb E_{\bsigma_{\tau}}[ \sigma_{\tau,i} \mid H_{\tau},\bsigma ]}
	\\
	& = 
 \mathbb{E}_{H_{\tau},\bsigma}\Bigsn{ \mathbb E_{\lambda_{\tau,i} \sim \text{Gam}(\alpha_{\tau,i},\beta_{\tau,i})}{ \sqrt{\frac{1}{ \kappa_{\tau,i}\lambda_{\tau,i}}}} \mid H_{\tau},\bsigma}
    \\
    &
    \leq \mathbb{E}_{H_{\tau},\bsigma}\Bigsn{ \sqrt{\frac{1}{ \kappa_{\tau,i}}}  \sqrt{\mathbb E_{\lambda_{\tau,i}} \biggsn{\frac{1}{\lambda_{\tau,i}} } } \mid H_{\tau},\bsigma} 
    \\
    &
    = \mathbb{E}_{H_{\tau},\bsigma}\Bigsn{ \sqrt{\frac{1}{ \kappa_{\tau,i}}}
    \sqrt{\frac{\beta_{\tau,i}}{ (\alpha_{\tau,i} -1)}   } \mid H_{\tau},\bsigma},
\end{align*}
where the last equality holds by the fact that given $H_{\tau}$, $\frac{1}{\lambda_{t,i}}$ follows the \textit{Inverse-Gamma}$(\alpha_{\tau,i},\beta_{\tau,i}^{-1})$ distribution and \cref{lem:inv_gam} \cite{cook08}. Now from the posterior update rules, given $H_{\tau}$ we have:
\begin{align*}
\frac{\beta_{\tau,i}}{ (\alpha_{\tau,i} -1)} 
    & \leq 
    \frac{\beta_{0,i} 
    +
    \frac{1}{2}\sum_{t \in [\tau]}\I{A_t = i}(x_{t,i} - \bar{x}_{\tau,i})^2
    +
    \frac{\kappa_{0,i} N_t(i) (\bar{x}_{\tau,i} - \mu_{0,i})^2 }{2(\kappa_{0,i} + N_{\tau}(i))} }{(\alpha_{0,i} + N_{\tau}(i)/2 -1)}
    \\
    & \leq 
    \frac{2\beta_{0,i} 
    +
    \sum_{t \in [\tau]}\I{A_t = i}(x_\tau(i) - \bar{x}_{\tau,i})^2
    +
    \kappa_{0,i} [(\bar{x}_{\tau,i} - \mu_{i})^2 + (\mu_{i} - \mu_{0,i})^2  ]  }
    {2(\alpha_{0,i} + N_{\tau}(i)/2 -1)}
\end{align*}

Now given any fixed reward-precision $\lambda_i \sim \text{Gam}(\alpha_{0,i},\beta_{0,i})$, $\sigma_{i}^2= \frac{1}{\lambda_{i}}$, mean-reward $\mu_i \sim \cN(\mu_{0,i},\sigma_{0,i}^2)$, and fixed number of pulls $N_\tau(i) = s$ (say), note $x_{t,i} \overset{iid}{\sim} \cN(\mu_{i}, \sigma_{i}^2)$ can be seen as $N_\tau(i) = s$ independent and identical draws from $\cN(\mu_i,\sigma_i^2)$. Thus
$\sum_{t = 1}^\tau\I{A_t = i} \frac{(x_{t,i} - \bar x_{\tau,i})^2}{\sigma_{i}^2} = \sum_{t = 1}^s\frac{(x_{t,i} - \bar x_{\tau,i})^2}{\sigma_{i}^2}$ follows $\chi^2(s-1)$, i.e. chi-squared distribution with $s-1$ degrees of freedom (see \cref{lem:avg_chi}). Further using \cref{lem:avg_chi}, part-(b) again, we know 
\[
E_{x_{t,i} \sim \cN(\mu_i,\sigma_i^2)} \biggsn{\sum_{t = 1}^s\frac{(x_{t,i} - \bar x_{\tau,i})^2}{\sigma_{i}^2}} = (s-1).
\]

Applying this to above chain of equations, note we get:

\begin{align*}
& E_{H_\tau,\bmu,\bsigma}\biggsn{\frac{\beta_{\tau,i}}{ (\kappa_{\tau,i})(\alpha_{\tau,i} -1)} \mid N_{\tau}(i)}
	\\    
    & = 
    E_{H_\tau,\bmu,\bsigma}\biggsn{
    \frac{{2\beta_{0,i} 
    +
    E_{x_{t,i}\sim \cN(\mu_i,\sigma_i^2)}\biggsn{\sigma_i^2 \sum_{t = 1}^{N_{\tau}(i)}(\frac{x_{t,i} - \bar{x}_{\tau,i}}{\sigma_i})^2 \mid \mu_i, \sigma_{i}}
    +
    \kappa_{0,i} [(\bar{x}_{\tau,i} - \mu_{i})^2 + (\mu_{i} - \mu_{0,i})^2  ]  }}
    {2(\kappa_{\tau,i})(\alpha_{0,i} + N_{\tau}(i)/2 -1)} \mid N_\tau(i)}
    \\
    & = 
    E_{H_\tau,\bmu,\bsigma}\biggsn{
    \frac{\Bigsn{2\beta_{0,i} 
    +
    \sigma_i^2 (N_{\tau}(i)-1)
    +
    \kappa_{0,i} \big[E_{\bar x_{\tau,i} \overset{\text{iid}}{\sim} \cN(\mu,\sigma^2/N_\tau(i))}[(\bar{x}_{\tau,i} - \mu_{i})^2 \mid \mu_i,\sigma_i] + (\mu_{i} - \mu_{0,i})^2  \big]  }}
    {2(\kappa_{\tau,i})(\alpha_{0,i} + N_{\tau}(i)/2 -1)} \mid N_\tau(i)}
    \\
    & \overset{(1)}{=} 
    E_{H_\tau,\bmu,\bsigma}\biggsn{
    \frac{\Bigsn{2\beta_{0,i} 
    +
    \sigma_i^2 (N_{\tau}(i)-1)
    +
    \kappa_{0,i} [\frac{\sigma_{i}^2}{N_{t,i}} + (\mu_{i} - \mu_{0,i})^2  ]  }}
    {2(\kappa_{\tau,i})(\alpha_{0,i} + N_{\tau}(i)/2 -1)} \mid N_\tau(i)}
\end{align*}
where (1) again follows since given $\mu_i, \sigma_i$ and $N_\tau(i)$ as before, $x_{t,i} \overset{\text{iid}}{\sim} \cN(\mu,\sigma^2)$, and thus by definition 
$\bar x_{\tau, i} = \frac{1}{N_{\tau}(i)} \sum_{t = 1}^\tau \I{A_t = i} x_{t,i} \sim \cN\big(\mu, \frac{\sigma^2}{N_{\tau}(i)}\big)$> This leads to $E_{\bar x_{\tau,i} \overset{\text{iid}}{\sim} \cN(\mu,\sigma^2/N_\tau(i))}[\bar x_{\tau, i} - \mu_i]^2 = \frac{\sigma_i^2}{N_{\tau}(i)}$. Thus the claim is concluded. 
\end{proof}


%% file: Lemmas.tex
\section{Technical Lemmas}
\label{sec:technical lemmas}

\begin{lem}
\label{lem:gaussian posterior update} 
Let $\mu_0 \sim \cN(\hat{\mu}, \hat{\sigma}^2)$, $\theta \mid \mu_0 \sim \cN(\mu_0, \sigma_0^2)$, and $Y_i \mid \theta \sim \cN(\theta, \sigma^2)$ for all $i \in [n]$. Then
\begin{align*}
  \mu_0 \mid Y_1, \dots, Y_n
  \sim \cN\left(\lambda^{-1} \left(\frac{\hat{\mu}}{\hat{\sigma}^2} +
  \frac{\sum_{i = 1}^n Y_i}{n \sigma_0^2 + \sigma^2}\right), \,
  \lambda^{-1}\right)\,, \quad
  \lambda
  = \frac{1}{\hat{\sigma}^2} + \left(\sigma_0^2 + \frac{\sigma^2}{n}\right)^{-1}\,.
\end{align*}
\end{lem}

\begin{proof}
The derivation is standard \citep{gelman07data} and we only include it for completeness. To simplify notation, let
\begin{align*}
  v
  = \sigma^{-2}\,, \quad
  v_0
  = \sigma_0^{-2}\,, \quad
  \hat{v}
  = \hat{\sigma}^{-2}\,, \quad
  c_1
  = v_0 + n v\,, \quad
  c_2
  = \hat{v} + v_0 - c_1^{-1} v_0^2\,.
\end{align*}
The posterior distribution of $\mu_0$ is
\begin{align*}
  & \int_\theta \left(\prod_{i = 1}^n \cN(Y_i; \theta, \sigma^2)\right)
  \cN(\theta; \mu_0, \sigma_0^2) \dif \theta \,
  \cN(\mu_0; \hat{\mu}, \hat{\sigma}^2) \\
  & \quad \propto \int_\theta \exp\left[
  - \frac{1}{2} v \sum_{i = 1}^n (Y_i - \theta)^2 -
  \frac{1}{2} v_0 (\theta - \mu_0)^2\right] \dif \theta
  \exp\left[- \frac{1}{2} \hat{v} (\mu_0 - \hat{\mu})^2\right]\,.
\end{align*}
Let $f(\mu_0)$ denote the integral. We solve it as
\begin{align*}
  f(\mu_0)
  & = \int_\theta \exp\left[- \frac{1}{2}
  \left(v \sum_{i = 1}^n (Y_i^2 - 2 Y_i \theta + \theta^2) +
  v_0 (\theta^2 - 2 \theta \mu_0 + \mu_0^2)\right)\right] \dif \theta \\
  & \propto \int_\theta \exp\left[- \frac{1}{2}
  \left(c_1 \left(\theta^2 -
  2 c_1^{-1} \theta \left(v \sum_{i = 1}^n Y_i + v_0 \mu_0\right)\right) +
  v_0 \mu_0^2\right)\right] \dif \theta \\
  & = \int_\theta \exp\left[- \frac{1}{2}
  \left(c_1 \left(\theta -
  c_1^{-1} \left(v \sum_{i = 1}^n Y_i + v_0 \mu_0\right)\right)^2 -
  c_1^{-1} \left(v \sum_{i = 1}^n Y_i + v_0 \mu_0\right)^2 +
  v_0 \mu_0^2\right)\right] \dif \theta \\
  & = \exp\left[- \frac{1}{2}
  \left(- c_1^{-1} \left(v \sum_{i = 1}^n Y_i + v_0 \mu_0\right)^2 +
  v_0 \mu_0^2\right)\right] \\
  & \propto \exp\left[- \frac{1}{2}
  \left(- c_1^{-1} \left(2 v_0 \mu_0 v \sum_{i = 1}^n Y_i + v_0^2 \mu_0^2\right) +
  v_0 \mu_0^2\right)\right]\,.
\end{align*}
Now we chain all equalities and have
\begin{align*}
  f(\mu_0) \exp\left[- \frac{1}{2} \hat{v} (\mu_0 - \hat{\mu})^2\right]
  & \propto \exp\left[- \frac{1}{2}
  \left(- c_1^{-1} \left(2 v_0 \mu_0 v \sum_{i = 1}^n Y_i + v_0^2 \mu_0^2\right) +
  v_0 \mu_0^2 + \hat{v} (\mu_0^2 - 2 \mu_0 \hat{\mu} + \hat{\mu}^2)
  \right)\right] \\
  & \propto \exp\left[- \frac{1}{2}
  c_2 \left(\mu_0^2 -
  2 c_2^{-1} \mu_0 \left(c_1^{-1} v_0 v \sum_{i = 1}^n Y_i +
  \hat{v} \hat{\mu}\right)\right)\right] \\
  & \propto \exp\left[- \frac{1}{2}
  c_2 \left(\mu_0 -
  c_2^{-1} \left(c_1^{-1} v_0 v \sum_{i = 1}^n Y_i +
  \hat{v} \hat{\mu}\right)\right)^2\right]\,.
\end{align*}
Finally, note that
\begin{align*}
  c_1
  & = \frac{1}{\sigma_0^2} + \frac{n}{\sigma^2}
  = \frac{n \sigma_0^2 + \sigma^2}{\sigma_0^2 \sigma^2}\,, \\
  c_1^{-1} v_0 v
  & = \frac{\sigma_0^2 \sigma^2}{n \sigma_0^2 + \sigma^2} \frac{1}{\sigma_0^2 \sigma^2}
  = \frac{1}{n \sigma_0^2 + \sigma^2}\,, \\
  c_2
  & = \frac{1}{\hat{\sigma}^2} +
  \frac{1}{\sigma_0^2} \left(1 - \frac{c_1^{-1}}{\sigma_0^2}\right)
  = \frac{1}{\hat{\sigma}^2} +
  \frac{1}{\sigma_0^2} \left(1 - \frac{\sigma^2}{n \sigma_0^2 + \sigma^2}\right)
  = \frac{1}{\hat{\sigma}^2} + \left(\sigma_0^2 + \frac{\sigma^2}{n}\right)^{-1}
  = \lambda\,.
\end{align*}
This concludes the proof.
\end{proof}

\begin{lem}
\label{lem:reciprocal root sum} For any integer $n$ and $a \geq 0$,
\begin{align*}
  \sum_{i = 1}^n \frac{1}{\sqrt{i + a}}
  \leq 2 (\sqrt{n + a} - \sqrt{a})
  \leq 2 \sqrt{n}\,.
\end{align*}
\end{lem}

\begin{proof}
Since $1 / \sqrt{i + a}$ decreases in $i$, the sum can be bounded using integration as
\begin{align*}
  \sum_{i = 1}^n \frac{1}{\sqrt{i + a}}
  \leq \int_{x = a}^{n + a} \frac{1}{\sqrt{x}} \dif x
  = 2 (\sqrt{n + a} - \sqrt{a})\,.
\end{align*}
The inequality $\sqrt{n + a} - \sqrt{a} \leq \sqrt{n}$ holds because the square root has diminishing returns.
\end{proof}

\begin{lem}
\label{lem:reciprocal sum} For any integer $n$ and $a \geq 0$,
\begin{align*}
  \sum_{i = 1}^n \frac{1}{i + a}
  \leq \log(1 + n / a)\,.
\end{align*}
\end{lem}

\begin{proof}
Since $1 / (i + a)$ decreases in $i$, the sum can be bounded using integration as
\begin{align*}
  \sum_{i = 1}^n \frac{1}{i + a}
  \leq \int_{x = a}^{n + a} \frac{1}{x} \dif x
  = \log(n + a) - \log a
  = \log(1 + n / a)\,.
\end{align*}
\end{proof}

\begin{lem}
\label{lem:reciprocal_sumsq} 
For any integer $n$ and $a \geq 0$,
\begin{align*}
  \sum_{i = 1}^n \frac{1}{i + a}
  \leq \log(1 + n / a)\,.
\end{align*}
\end{lem}

%% file: mainArXiv-vts.bbl
\begin{thebibliography}{41}
\providecommand{\natexlab}[1]{#1}
\providecommand{\url}[1]{\texttt{#1}}
\expandafter\ifx\csname urlstyle\endcsname\relax
  \providecommand{\doi}[1]{doi: #1}\else
  \providecommand{\doi}{doi: \begingroup \urlstyle{rm}\Url}\fi

\bibitem[Abbasi-Yadkori et~al.(2011)Abbasi-Yadkori, Pal, and
  Szepesvari]{abbasi-yadkori11improved}
Yasin Abbasi-Yadkori, David Pal, and Csaba Szepesvari.
\newblock Improved algorithms for linear stochastic bandits.
\newblock In \emph{Advances in Neural Information Processing Systems 24}, pages
  2312--2320, 2011.

\bibitem[Agrawal and Goyal(2012)]{agrawal12analysis}
Shipra Agrawal and Navin Goyal.
\newblock Analysis of {Thompson} sampling for the multi-armed bandit problem.
\newblock In \emph{Proceeding of the 25th Annual Conference on Learning
  Theory}, pages 39.1--39.26, 2012.

\bibitem[Agrawal and Goyal(2013)]{agrawal13further}
Shipra Agrawal and Navin Goyal.
\newblock Further optimal regret bounds for {Thompson} sampling.
\newblock In \emph{Proceedings of the 16th International Conference on
  Artificial Intelligence and Statistics}, pages 99--107, 2013.

\bibitem[Audibert et~al.(2009{\natexlab{a}})Audibert, Munos, and
  Szepesv{\'a}ri]{audibert09}
Jean-Yves Audibert, R{\'e}mi Munos, and Csaba Szepesv{\'a}ri.
\newblock Exploration--exploitation tradeoff using variance estimates in
  multi-armed bandits.
\newblock \emph{Theoretical Computer Science}, 410\penalty0 (19):\penalty0
  1876--1902, 2009{\natexlab{a}}.

\bibitem[Audibert et~al.(2009{\natexlab{b}})Audibert, Munos, and
  Szepesvari]{audibert09exploration}
Jean-Yves Audibert, Remi Munos, and Csaba Szepesvari.
\newblock Exploration-exploitation tradeoff using variance estimates in
  multi-armed bandits.
\newblock \emph{Theoretical Computer Science}, 410\penalty0 (19):\penalty0
  1876--1902, 2009{\natexlab{b}}.

\bibitem[Auer et~al.(2002)Auer, Cesa-Bianchi, and Fischer]{auer02finitetime}
Peter Auer, Nicolo Cesa-Bianchi, and Paul Fischer.
\newblock Finite-time analysis of the multiarmed bandit problem.
\newblock \emph{Machine Learning}, 47:\penalty0 235--256, 2002.

\bibitem[Bubeck et~al.(2009)Bubeck, Munos, and Stoltz]{bubeck09pure}
Sebastien Bubeck, Remi Munos, and Gilles Stoltz.
\newblock Pure exploration in multi-armed bandits problems.
\newblock In \emph{Proceedings of the 20th International Conference on
  Algorithmic Learning Theory}, pages 23--37, 2009.

\bibitem[Chapelle and Li(2012)]{chapelle11empirical}
Olivier Chapelle and Lihong Li.
\newblock An empirical evaluation of {Thompson} sampling.
\newblock In \emph{Advances in Neural Information Processing Systems 24}, pages
  2249--2257, 2012.

\bibitem[Cook(2008)]{cook08}
John~D Cook.
\newblock Inverse gamma distribution.
\newblock \emph{online: http://www. johndcook. com/inverse gamma. pdf, Tech.
  Rep}, 2008.

\bibitem[Faella et~al.(2020)Faella, Finzi, and Sauro]{faella20kernel}
Marco Faella, Alberto Finzi, and Luigi Sauro.
\newblock Rapidly finding the best arm using variance.
\newblock In \emph{Proceedings of the 24th European Conference on Artificial
  Intelligence}, 2020.

\bibitem[Gabillon et~al.(2011)Gabillon, Ghavamzadeh, Lazaric, and
  Bubeck]{gabillon11multibandit}
Victor Gabillon, Mohammad Ghavamzadeh, Alessandro Lazaric, and Sebastien
  Bubeck.
\newblock Multi-bandit best arm identification.
\newblock In \emph{Advances in Neural Information Processing Systems 24}, 2011.

\bibitem[Garivier and Cappe(2011)]{garivier11klucb}
Aurelien Garivier and Olivier Cappe.
\newblock The {KL-UCB} algorithm for bounded stochastic bandits and beyond.
\newblock In \emph{Proceeding of the 24th Annual Conference on Learning
  Theory}, pages 359--376, 2011.

\bibitem[Gelman and Hill(2007)]{gelman07data}
Andrew Gelman and Jennifer Hill.
\newblock \emph{Data Analysis Using Regression and Multilevel/Hierarchical
  Models}.
\newblock Cambridge University Press, New York, NY, 2007.

\bibitem[Gopalan et~al.(2014)Gopalan, Mannor, and Mansour]{gopalan14thompson}
Aditya Gopalan, Shie Mannor, and Yishay Mansour.
\newblock Thompson sampling for complex online problems.
\newblock In \emph{Proceedings of the 31st International Conference on Machine
  Learning}, pages 100--108, 2014.

\bibitem[Honda and Takemura(2014)]{honda14optimality}
Junya Honda and Akimichi Takemura.
\newblock Optimality of {Thompson} sampling for {Gaussian} bandits depends on
  priors.
\newblock In \emph{Proceedings of the 17th International Conference on
  Artificial Intelligence and Statistics}, 2014.

\bibitem[Hong et~al.(2022{\natexlab{a}})Hong, Kveton, Katariya, Zaheer, and
  Ghavamzadeh]{hong22deep}
Joey Hong, Branislav Kveton, Sumeet Katariya, Manzil Zaheer, and Mohammad
  Ghavamzadeh.
\newblock Deep hierarchy in bandits.
\newblock In \emph{Proceedings of the 39th International Conference on Machine
  Learning}, 2022{\natexlab{a}}.

\bibitem[Hong et~al.(2022{\natexlab{b}})Hong, Kveton, Zaheer, and
  Ghavamzadeh]{hong22hierarchical}
Joey Hong, Branislav Kveton, Manzil Zaheer, and Mohammad Ghavamzadeh.
\newblock Hierarchical {Bayesian} bandits.
\newblock In \emph{Proceedings of the 25th International Conference on
  Artificial Intelligence and Statistics}, 2022{\natexlab{b}}.

\bibitem[Jourdan et~al.(2022)Jourdan, Degenne, and Kaufmann]{jourdan22dealing}
Marc Jourdan, Remy Degenne, and Emilie Kaufmann.
\newblock Dealing with unknown variances in best-arm identification.
\newblock \emph{CoRR}, abs/2210.00974, 2022.
\newblock URL \url{https://arxiv.org/abs/2210.00974}.

\bibitem[Kim et~al.(2022)Kim, Yang, and Jun]{kim22}
Yeoneung Kim, Insoon Yang, and Kwang-Sung Jun.
\newblock Improved regret analysis for variance-adaptive linear bandits and
  horizon-free linear mixture mdps.
\newblock \emph{Advances in Neural Information Processing Systems},
  35:\penalty0 1060--1072, 2022.

\bibitem[Kveton et~al.(2021)Kveton, Konobeev, Zaheer, Hsu, Mladenov, Boutilier,
  and Szepesvari]{kveton21metathompson}
Branislav Kveton, Mikhail Konobeev, Manzil Zaheer, Chih-Wei Hsu, Martin
  Mladenov, Craig Boutilier, and Csaba Szepesvari.
\newblock Meta-{Thompson} sampling.
\newblock In \emph{Proceedings of the 38th International Conference on Machine
  Learning}, 2021.

\bibitem[Lai and Robbins(1985)]{lai85asymptotically}
T.~L. Lai and Herbert Robbins.
\newblock Asymptotically efficient adaptive allocation rules.
\newblock \emph{Advances in Applied Mathematics}, 6\penalty0 (1):\penalty0
  4--22, 1985.

\bibitem[Lai(1987)]{lai87adaptive}
Tze~Leung Lai.
\newblock Adaptive treatment allocation and the multi-armed bandit problem.
\newblock \emph{The Annals of Statistics}, 15\penalty0 (3):\penalty0
  1091--1114, 1987.

\bibitem[Lalitha et~al.(2023)Lalitha, Kalantari, Ma, Deoras, and
  Kveton]{lalitha2023fixed}
Anusha~Lalitha Lalitha, Kousha Kalantari, Yifei Ma, Anoop Deoras, and Branislav
  Kveton.
\newblock Fixed-budget best-arm identification with heterogeneous reward
  variances.
\newblock In \emph{Uncertainty in Artificial Intelligence}, pages 1164--1173.
  PMLR, 2023.

\bibitem[Lattimore and Szepesvari(2019)]{lattimore19bandit}
Tor Lattimore and Csaba Szepesvari.
\newblock \emph{Bandit Algorithms}.
\newblock Cambridge University Press, 2019.

\bibitem[Lattimore and Szepesv{\'a}ri(2020)]{lattimore2020bandit}
Tor Lattimore and Csaba Szepesv{\'a}ri.
\newblock \emph{Bandit algorithms}.
\newblock Cambridge University Press, 2020.

\bibitem[Lu et~al.(2021)Lu, Tao, and Zhang]{lu21variancedependent}
Pinyan Lu, Chao Tao, and Xiaojin Zhang.
\newblock Variance-dependent best arm identification.
\newblock In \emph{Proceedings of the 37th Conference on Uncertainty in
  Artificial Intelligence}, 2021.

\bibitem[Lu and {Van Roy}(2019)]{lu19informationtheoretic}
Xiuyuan Lu and Benjamin {Van Roy}.
\newblock Information-theoretic confidence bounds for reinforcement learning.
\newblock In \emph{Advances in Neural Information Processing Systems 32}, 2019.

\bibitem[Mukherjee et~al.(2018)Mukherjee, Naveen, Sudarsanam, and
  Balaraman]{mukherjee18}
Subhojyoti Mukherjee, KP~Naveen, Nandan Sudarsanam, and Ravindran Balaraman.
\newblock Efficient-ucbv: An almost optimal algorithm using variance estimates.
\newblock \emph{AAAI Conference on Artificial Intelligence}, 32, 2018.

\bibitem[Murphy(2007)]{murphy2007conjugate}
Kevin~P Murphy.
\newblock Conjugate bayesian analysis of the gaussian distribution.
\newblock \emph{def}, 1\penalty0 (2$\sigma$2):\penalty0 16, 2007.

\bibitem[Pearson(1936)]{pearson36method}
Karl Pearson.
\newblock Method of moments and method of maximum likelihood.
\newblock \emph{Biometrika}, 28\penalty0 (1/2):\penalty0 34--59, 1936.

\bibitem[Russo and {Van Roy}(2014)]{russo14learning}
Daniel Russo and Benjamin {Van Roy}.
\newblock Learning to optimize via posterior sampling.
\newblock \emph{Mathematics of Operations Research}, 39\penalty0 (4):\penalty0
  1221--1243, 2014.

\bibitem[Russo and {Van Roy}(2016)]{russo16information}
Daniel Russo and Benjamin {Van Roy}.
\newblock An information-theoretic analysis of {Thompson} sampling.
\newblock \emph{Journal of Machine Learning Research}, 17\penalty0
  (68):\penalty0 1--30, 2016.

\bibitem[Russo et~al.(2018)Russo, {Van Roy}, Kazerouni, Osband, and
  Wen]{russo18tutorial}
Daniel Russo, Benjamin {Van Roy}, Abbas Kazerouni, Ian Osband, and Zheng Wen.
\newblock A tutorial on {Thompson} sampling.
\newblock \emph{Foundations and Trends in Machine Learning}, 11\penalty0
  (1):\penalty0 1--96, 2018.

\bibitem[Saha et~al.(2020)Saha, Gaillard, and Valko]{SGV20}
Aadirupa Saha, Pierre Gaillard, and Michal Valko.
\newblock Improved sleeping bandits with stochastic action sets and adversarial
  rewards.
\newblock In \emph{International Conference on Machine Learning}, pages
  8357--8366. PMLR, 2020.

\bibitem[Thompson(1933)]{thompson33likelihood}
William~R. Thompson.
\newblock On the likelihood that one unknown probability exceeds another in
  view of the evidence of two samples.
\newblock \emph{Biometrika}, 25\penalty0 (3-4):\penalty0 285--294, 1933.

\bibitem[Wan et~al.(2021)Wan, Ge, and Song]{wan21metadatabased}
Runzhe Wan, Lin Ge, and Rui Song.
\newblock Metadata-based multi-task bandits with {Bayesian} hierarchical
  models.
\newblock In \emph{Advances in Neural Information Processing Systems 34}, 2021.

\bibitem[Zhang et~al.(2021)Zhang, Yang, Ji, and Du]{zhang21}
Zihan Zhang, Jiaqi Yang, Xiangyang Ji, and Simon~S Du.
\newblock Improved variance-aware confidence sets for linear bandits and linear
  mixture mdp.
\newblock \emph{Advances in Neural Information Processing Systems},
  34:\penalty0 4342--4355, 2021.

\bibitem[Zhao et~al.(2022)Zhao, Zhou, He, and Gu]{quan22}
Heyang Zhao, Dongruo Zhou, Jiafan He, and Quanquan Gu.
\newblock Bandit learning with general function classes: Heteroscedastic noise
  and variance-dependent regret bounds.
\newblock 2022.

\bibitem[Zhao et~al.(2023)Zhao, He, Zhou, Zhang, and Gu]{quan23}
Heyang Zhao, Jiafan He, Dongruo Zhou, Tong Zhang, and Quanquan Gu.
\newblock Variance-dependent regret bounds for linear bandits and reinforcement
  learning: Adaptivity and computational efficiency.
\newblock \emph{arXiv preprint arXiv:2302.10371}, 2023.

\bibitem[Zhou and Tian(2022)]{zhou22approximate}
Ruida Zhou and Chao Tian.
\newblock Approximate top-$m$ arm identification with heterogeneous reward
  variances.
\newblock In \emph{Proceedings of the 25th International Conference on
  Artificial Intelligence and Statistics}, 2022.

\bibitem[Zhu and Tan(2020)]{MV20}
Qiuyu Zhu and Vincent Tan.
\newblock Thompson sampling algorithms for mean-variance bandits.
\newblock In \emph{International Conference on Machine Learning}, pages
  11599--11608. PMLR, 2020.

\end{thebibliography}
